\documentclass[a4paper,fleqn]{cas-sc}

\usepackage[authoryear,longnamesfirst]{natbib}

\usepackage{amsmath}
\usepackage{amsfonts}
\usepackage{multirow}
\usepackage{graphicx}
\usepackage{tikz}
\usepackage{natbib}
\usepackage{amsthm}
\usepackage{amssymb}
\usepackage{comment}
\usepackage{url}
\usepackage{hyperref}
\usepackage{enumerate}
\usepackage{cases}
\usepackage[ruled,vlined,linesnumbered]{algorithm2e}
\usepackage{makecell}

\usepackage{xcolor, colortbl}
\newcommand{\struc}[1]{{\color{blue} #1}}

\newcommand{\Vector}[1]{\ensuremath{\boldsymbol{#1}}}
\newcommand{\Q}{\mathbb{Q}}
\newcommand{\bP}{\mathbb{P}}
\newcommand{\R}{\mathbb{R}}
\newcommand{\C}{\mathbb{C}}
\newcommand{\Z}{\mathbb{Z}}
\newcommand{\LX}{{{\mathcal V}(\Vector{f})}}

\newcommand{\pX}{\prec_{\Vector{x}}}
\newcommand{\ta}{\tau_{\Vector{a}}}

\def \coeff {{\rm coeff}}

\def \lc {{\rm lc}}
\def \cM { {\mathcal M} }
\def \vf {\Vector{f}}
\def \vG {\Vector{G}}

\def \vU {\Vector{u}}
\def \vV {\Vector{v}}
\def \vX {\Vector{x}}
\def \vY {\Vector{y}}
\def \vB {\Vector{b}}
\def \vA {\Vector{a}}
\def \mV {\mathcal{V}}
\def \vp {\Vector{p}}
\def \mI {\mathcal{I}}
\def \mH {\mathcal{H}}
\def \mK {\mathcal{K}}

\def \E {\eta}
\newtheorem{theorem}{Theorem}[section]
\newtheorem{lemma}[theorem]{Lemma}

\newtheorem{definition}[theorem]{Definition}
\newtheorem{example}[theorem]{Example}

\newtheorem{corollary}[theorem]{Corollary}

\newtheorem{remark}[theorem]{Remark}
\newtheorem{model}[theorem]{Model}
\newtheorem{assumption}[theorem]{Assumption}

\numberwithin{equation}{section}

\def\tsc#1{\csdef{#1}{\textsc{\lowercase{#1}}\xspace}}
\tsc{WGM}
\tsc{QE}
\tsc{EP}
\tsc{PMS}
\tsc{BEC}
\tsc{DE}

\ExplSyntaxOn
\cs_gset:Npn \__first_footerline:
{
  \group_begin:
  \small
  \sffamily
  \ifnum\theblind>0\relax
  \else
    \__short_authors:
  \fi
  \group_end:
}
\ExplSyntaxOff

\begin{document}
\let\WriteBookmarks\relax
\def\floatpagepagefraction{1}
\def\textpagefraction{.001}
\shorttitle{Detecting Nonproperness of Likelihood Equations}
\shortauthors{X. Tang et~al.}

\title [mode = title]{Detecting Nonproperness of Likelihood Equations}



\author[1]{Xiaoxian Tang}
\cormark[1]
\ead{xiaoxian@buaa.edu.cn}
\ead[url]{xiaoxian-tang.github.io/website/}


\affiliation[1]{organization={School of Mathematical Sciences, Beihang University},
                addressline={Xueyuan Road},
                postcode={100191},
                city={Beijing},
                country={China}}

\author[2]{Bican Xia}

\ead{xbc@math.pku.edu.cn}


\affiliation[2]{organization={School of Mathematical Sciences, Peking University},
                addressline={Yiheyuan Road},
                postcode={100871},
                city={Beijing},
                country={China}}

\author[3]{Tianqi Zhao}
\cormark[1]
\ead{zhaotq@zgclab.edu.cn}

\affiliation[3]{organization={Zhongguancun Laboratory},
                addressline={Zhongguancun East Road
},
                postcode={100084},
                city={Beijing},
                country={China}}

\cortext[cor1]{Corresponding author}


\begin{abstract}
Given an algebraic statistical model, a challenging problem is classifying the data according to the number of positive critical points of the likelihood function. The positive critical points are the positive solutions to an algebraic system, say likelihood equations. So, identifying the number of positive critical points is a real root classification problem for the likelihood equations.
A discriminant variety of a likelihood-equation system geometrically describes the data for which  the number of real solutions becomes
unusual. As an essential component of the discriminant variety,
the nonproperness set collects the data such that the likelihood-equation system has a solution at infinity. So, the number of real solutions varies when the data passes the nonproperness set, and identifying the nonproperness set
 plays a crucial role in the real root classification.
In this work, we develop a novel method for computing nonproperness sets of
likelihood-equation systems. We prove the correctness of this method. We show experimentally that it is far more efficient than the known methods in the literature.
\end{abstract}



\begin{keywords}
Real root classification \sep Nonproperness \sep General zero-dimensional system \sep Likelihood equation
\end{keywords}

\maketitle

\section{Introduction}\label{sec:intro}
{\bf Background} This work is motivated by the real root classification problem of the following likelihood equations
\begin{align}\label{eq:lle}
\begin{array}{rl}
f_{1}(\Vector{u}, \Vector{p}, \Vector{\lambda})&:=\;p_1\cdot \left(\lambda_{1}+\frac{\partial g_1}{\partial p_1}\lambda_2+\cdots+\frac{\partial g_s}{\partial p_1}\lambda_{s+1}\right)-u_1,\\
&\;\;\vdots\\
f_{n}(\Vector{u},\Vector{p}, \Vector{\lambda})&:=\; p_n\cdot \left(\lambda_{1}+\frac{\partial g_1}{\partial p_n}\lambda_2+\cdots+\frac{\partial g_s}{\partial
p_n}\lambda_{s+1}\right)-u_n,\\
f_{n+1}(\Vector{u},\Vector{p},\Vector{\lambda})&:=\; g_1(p_1, \ldots,p_n),\\
&\;\;\vdots\\
f_{n+s}(\Vector{u},\Vector{p},\Vector{\lambda}) &:=\; g_s(p_1,\ldots,p_n),\\
f_{n+s+1}(\Vector{u},\Vector{p},\Vector{\lambda})&:=\; 
p_1+\cdots+p_n-1,
\end{array}
\end{align}where the above  parametric polynomial  system $\Vector{f}:=(f_1, \ldots, f_{n+s+1})$ contains probability variables $\Vector{p}:=(p_1, \ldots, p_n)$, Lagrange multipliers $\Vector{\lambda}:=(\lambda_1, \ldots, \lambda_{s+1})$, and parameters $\Vector{u}:=(u_1, \ldots, u_n)$ representing the data obtained from statistical experiments. It is well known that all critical points of the maximum likelihood estimation (MLE) problem are real solutions to the likelihood equations \eqref{eq:lle} \citep{ABBGHHNRS2017, AGKMS2024, BHR2007, CHKS2006, GDP2012, EJ2014, HRS, SAB2005}. Classifying the parameters (i.e., the data) according to the number of positive solutions of likelihood equations (i.e., the positive critical points of the MLE problem) is an important and a challenging problem since  it is a specific real quantifier elimination problem \citep{BPR1996, BPRRoadmap, BPRBook,  collins1975, ch1991,  SS2003, tarski1951}.  In principle, the real root classification problem can be carried out automatically using some known software systems,  e.g., \citep{brown2003, CDMMX2010, DS1997, DV2005, BP2001}. However, the expense of these tools is way beyond current computing capabilities since the likelihood-equation systems usually contain a lot of variables and parameters even for small statistical models \citep{BPR1996,grig88,renegar1992-1,renegar1992-2,renegar1992-3}.

One key step of real root classification is to compute the discriminant variety, which consists of two  core components: the discriminant locus and the nonproperness set, denoted by $\LX_{J}$ and $\LX_{\infty}$, respectively (see the precise definitions later in Definition \ref{def:nddv}). Both components are algebraically closed. The number of real solutions only changes when the parameters  pass the discriminant variety $\LX_{J}\cup \LX_{\infty}$.
It is well known that when the parameters pass the discriminant locus, at least two simple real solutions might merge into a multiple one.   When the parameters pass the nonproperness set, some real solutions might go to infinity \citep{SS2004}.  So, the number of real solutions might be ``unusual'' when the parameters are exactly located in the discriminant variety, see the following example.

\begin{example}\label{ex:symmetric}
Consider the algebraic model given by a $3\times 3$ symmetric matrix, see the statistical story in  \cite[Section 5.2]{Tang2017}.
\begin{equation*}
    S\;=\;\begin{bmatrix}
    2p_{2} & p_{1} & p_{3}\\
    p_{1}  & 2p_{4} & p_{5}\\
    p_{3}  & p_{5} & 2p_{6}
    \end{bmatrix}.
\end{equation*}
Suppose the vectors of the probabilities from the simplex
$\triangle_5=\{(p_{1},\ldots,p_{6})\in\R^{6}_{>0}\mid p_{1}+p_{2}+p_{3}+p_{4}+p_{5}+p_{6}=1\}$ satisfies the constraint
$\det S=0$.
For any given data-vector $(u_{1},u_{2},u_{3},u_{4},u_{5},u_{6})\in\R^{6}_{\geq 0}$, where $u_{i}$ denotes the data corresponding to  $p_{i}$,
the goal is to maximize  the
likelihood function
${\prod\limits_{1\le i\le6} p_{i}^{u_{i}}}$.
In order to find all real critical points to the likelihood function, we form the following system of likelihood equations by applying the Lagrange multiplier method:
\begin{align}\label{eq:ex1}
f_1&=\;p_{1}(\lambda_1 + (2p_{3}p_{5}-4p_{1}p_{6})\lambda_2) -u_{1} =0\notag\\
f_2&=\;p_{2}(\lambda_1+(8p_{4}p_{6}-2p_{5}^2)\lambda_2)-u_{2}=0 \notag\\
f_3&=\;p_{3}(\lambda_1 + (2p_{1}p_{5}-4p_{3}p_{4})\lambda_2) -u_{3} = 0 \notag \\
f_4&=\;p_{4}(\lambda_1 + (8p_{2}p_{6}-2p^2_{3})\lambda_2)-u_{4} =0 \notag\\
f_5&=\;p_{5}(\lambda_1 + (2p_{1}p_{3}-4p_{2}p_{5})\lambda_2)-u_{5} = 0 \notag \\
f_6&=\;p_{6}(\lambda_1 + (8p_{2}p_{4}-2p^2_{1})\lambda_2)-u_{6} = 0\notag\\
f_7&=\;\det S=0 \notag\\
f_8&=\;p_{1}+p_{2}+p_{3}+p_{4}+p_{5}+p_{6}-1=0
\end{align}
where $\lambda_1$ and $\lambda_2$ are two Lagrange multipliers.
In the above algebraic system, we view $p_{i},\lambda_1$ and $\lambda_2$ as unknowns, and we view $u_{i}$ as parameters. That means we have a square algebraic system
with 8 equations in 8 unknowns and 6 parameters.
Notice that
The \struc{\textit{maximum-likelihood-degree}} (\struc{\textit{ML-degree}}) of an algebraic statistical model is defined as the number of complex solutions to its likelihood equations for a generic choice of data. Moreover, it typically serves as a rough indicator of the computational complexity of numerical solvers based on homotopy continuation for a given model. Here, for this example, the ML-degree is 6 \citep{SAB2005}.
By \citep{Tang2017},
the nonproperness set $\LX_{\infty}$ is the algebraic variety generated by the  single polynomial
\begin{align}\label{eq:symmetric}
    w \;=\;&(u_1^2u_6-u_1u_3u_5-4u_2u_4u_6+u_2u_5^2+u_3^2u_4)\notag\\
    &\cdot(u_1+u_2+u_4)\cdot(u_2+u_3+u_6)\notag\\
      &\cdot(u_4+u_5+u_6)\cdot(u_1+u_2+u_3+u_4+u_5+u_6)\notag\\
      &\cdot(u_1+2u_2+u_3)\cdot(u_1+2u_4+u_5)\cdot(u_3+u_5+2u_6),
\end{align}
and the discriminant locus $\LX_{J}$ is the algebraic variety generated by a single polynomial $DX_J$ with $1307$ terms and total degree $12$, where
{\scriptsize
\begin{align}\label{eq:ex1dj}
DX_J\;=\;
&-64u_{2}^5u_{4}^3u_{5}^4-384u_{2}^5u_{4}^3u_{5}^3u_{6}-768u_{2}^5u_{4}^3u_{5}^2
u_{6}^2-512u_{2}^5u_{4}^3u_{5}u_{6}^3-96u_{2}^5u_{4}^2u_{5}^5-576u_{2}^5u_{4}^2u_{5}^4u_{6}\notag\\
&-1152u_{2}^5u_{4}^2u_{5}^3u_{6}^2-768u_{2}^5u_{4}^2u_{5}^2u_{6}^3-48u_{2}^5u_{4}u_{5}^6-
288u_{2}^5u_{4}u_{5}^5u_{6}-576u_{2}^5u_{4}u_{5}^4u_{6}^2-384u_{2}^5u_{4}u_{5}^3u_{6}^3\notag\\
&-8u_{2}^5u_{5}^7-48u_{2}^5u_{5}^6u_{6}-96u_{2}^5u_{5}^5u_{6}^2-64u_{2}^5u_{5}^4u_{6}^3+48
u_{2}^4u_{1}^2u_{4}^2u_{5}^4+288u_{2}^4u_{1}^2u_{4}^2u_{5}^3u_{6}\notag\\
&+\cdots+48u_{3}^5u_{4}^4u_{5}^2u_{6}-6u_{3}^5u_{4}^3u_{5}^4-64u_{3}^4u_{4}^5u_{6}^3+48u_{3}^4u_{4}^4u_{5}^2u_{6}^2-12u_{3}^4u_{4}^3u_{5}^4u_{6}+u_{3}^4u_{4}^2u_{5}^6.
\end{align}}For the algebraic system \eqref{eq:ex1},
the number of complex solutions for the generic parameters is $6$, and
the number of real solutions for the parameters from the complement of $\LX_{J}\cup \LX_{\infty}$ (an open region) can be $2$ or $6$. However,
the number of real solutions can be $1$ or infinite for the parameters located in the nonproperness set $\LX_{\infty}$.
\end{example}
{\bf Challenge} The standard methods for computing $\LX_{J}$ and $\LX_{\infty}$ were based on computing Gr\"obner bases \citep{DV2005}. In \citep{RT2015,Tang2017}, the authors proposed  a novel approach  for efficiently computing the generator of the discriminant variety for a likelihood-equation system by a specialization/interpolation method.
Especially, in \citep{TWZ2019}, the  interpolation method was further improved for computing the generator of the discriminant locus $\LX_{J}$.
 However, how to efficiently compute the nonproperness set $\LX_{\infty}$ is rarely studied.
 The experiments show that either the Gr\"onber-bases methods or the interpolation methods are not applicable even for middle-sized   algebraic models, see the columns ``Standard", ``Interpolation I" and ``Interpolation II" in Table \ref{table:literatureOld}. 
 In the standard method \citep{DV2005},
 no matter  for computing $\LX_{J}$ or computing $\LX_{\infty}$,
 the first step is always to compute
 Gr\"onber bases. The difference is that the monomial ordering taken for computing $\LX_{J}$ is an elimination order such as the lexicographic order and that for comupting  $\LX_{\infty}$ is a certain admissible block order, where the variables have a degree reverse lexicographic  ordering, see \cite[Algorithm PROPERNESSDEFECTS]{DV2005}. So, in principle, for small problems,
 computing $\LX_{\infty}$ is faster than computing
 $\LX_{J}$, and
   when the size of the problem is large, the standard method can not complete the computation for both.
 On the other hand, the interpolation methods \citep{RT2015,Tang2017,TWZ2019} are based on a linear-lifting technique, and so, their efficiency mainly depends on  the degree of the objective polynomial.    As a result, the computational timings consumed by  interpolating
    $\LX_{\infty}$ and $\LX_{J}$
 are also usually similar if the generator polynomials of $\LX_{\infty}$ and $\LX_{J}$  have similar degrees. For instance, in Example \ref{ex:symmetric},
    the degrees of  \eqref{eq:symmetric} and \eqref{eq:ex1dj} are $10$ and $12$, which are similar. So, the interpolation methods also  take long time to compute
  \eqref{eq:symmetric}, even though the polynomial \eqref{eq:symmetric} looks much simpler (e.g., it is factorable) than  \eqref{eq:ex1dj}.

{\bf Problem Statement} The goal of this paper is to efficiently compute  the nonproperness set $\LX_{\infty}$ for a given system of likelihood equations. In our setting, the likelihood-equation system is a general zero-dimensional system, and it is a standard assumption  that the nonproperness set has codimension one \citep{DV2005}. More formally, we have the following problem statement:
\\
Input: Likelihood equations $$f_1, \ldots, f_{n+s+1}\in {\mathbb Q}[u_1,\ldots, u_n, p_1, \ldots, p_{n}, \lambda_1, \ldots, \lambda_{s+1}].$$
Output: A generator of the nonproperness set, which is a polynomial in ${\mathbb Q}[u_1,\ldots, u_n]$.

{\bf Contribution} We list our main contributions as follows.
\begin{itemize}
\item[(a)] We prove that for a generic  specialization, the ideal generated by the nonproperness set of a specialized likelihood system is equal to
the specialization of the ideal generated by the nonproperness set of the likelihood system; see Theorem \ref{thm:main}.
    \item[(b)] We propose a probabilistic algorithm for computing nonproperness sets of Lagrange likelihood equations; see Algorithm \ref{alg:new}.
    We implemented the algorithm with {\tt Maple2023}.
Our experiments show that Algorithm \ref{alg:new} is significantly more efficient than 
the standard approach \citep{DV2005} 
(Algorithm \ref{alg:standard}) and the interpolation methods proposed in \citep{Tang2017}.
For instance, the largest model we can currently compute has ML-degree $14$ (i.e., Model \ref{ex:l8} in Appendix \ref{sec:appendix}), which can not be tackled by the known symbolic methods.
While the standard method gets out of memory for computing this model,
the estimated timing of the interpolation method proposed in \citep{Tang2017} is $7347$ days.
See more comparisons  on the computational timings for different methods in Table \ref{table:literatureOld}.
\end{itemize}
\begin{table}[b]
\small
\centering
\scalebox{0.85}{
\begin{tabular}{|c|c||c|c|c|c||c|c|c|} \hline

\multirow{2}{*}{Models} &
\multirow{2}{*}{MLD} &


\multicolumn{4}{c||}{Timings} &
\multicolumn{3}{c|}{Samplings} 
\\\cline{3-9}


 & &{\footnotesize Standard}  &{\footnotesize Interpolation I}& {\footnotesize Interpolation II}  &\cellcolor{red!25} {\footnotesize New Method} & {\footnotesize Interpolation I} & {\footnotesize Interpolation II}&\cellcolor{red!25} {\footnotesize New Method}\\ \hline

\ref{ex:l1} &2 &1.813 s  &OOT    &81.25 s &\cellcolor{red!25} {\bf 0.130} s  &381  &625   &\cellcolor{red!25}{\bf 1}\\
\ref{ex:l2} &3 &0.465 s  &1.229 s   &3.240 s &\cellcolor{red!25}{\bf 0.205} s &16   &20     &\cellcolor{red!25}{\bf 3} \\
\ref{ex:l3} &4 &17.992 s  &OOT    &422.576 s &\cellcolor{red!25} {\bf 0.092} s
&1,050  &2,401   &\cellcolor{red!25}{\bf 1}\\ 
\ref{ex:l4} &6 &OOM       &OOT  &1.309 h & \cellcolor{red!25}{\bf 67.603} s  &721  &1,521   &\cellcolor{red!25}{\bf 19} \\ 

\ref{ex:l5} &9 &OOM      &OOT     & OOT  &\cellcolor{red!25} {\bf 1.776} h
&25,536  &518,400 
& \cellcolor{red!25}{\bf 209}\\ 

\ref{ex:l6} &10&OOM     &OOT   &OOT 
&\cellcolor{red!25} {\bf 890.656} s
&502,440 &4,782,969  &\cellcolor{red!25}{\bf 57} \\ 

\ref{ex:l7} &12&OOM        &OOT  &OOT 
 &\cellcolor{red!25} {\bf 3.755} h &684&1,300&\cellcolor{red!25} {\bf 77} \\ 

\ref{ex:l8} &14&OOM    &OOT    &OOT 
 &\cellcolor{red!25}{\bf 6.488} h
&142,065 &1,000,000 &\cellcolor{red!25} {\bf 22}\\ 

\ref{ex:l9} &23&OOM     &OOT  &OOT 
 &\cellcolor{red!25} OOT &2,128 &4,046 &\cellcolor{red!25} {\bf 593} 
\\ \hline
\end{tabular}}
\caption{Timings for computing nonproperness sets\\
 {\it Note: We compare the standard method: Algorithm \ref{alg:standard}
\citep{DV2005} via a regular Gr\"obner basis computation \citep{FGb}, the interpolation method with two strategies \citep{Tang2017} and the new  method (Algorithm \ref{alg:new}).}}
\label{table:literatureOld}
\begin{center}
\begin{itemize}
\item[{\bf Models:}] All the testing statistical models \ref{ex:l1}--\ref{ex:l9} are listed in Appendix \ref{sec:appendix}.
\item[{\bf MLD:}] The models are ordered by their ML-degrees.
\item[{\bf Timings:}]  ``Timings'' records the computation timings for each method.\\
(s: seconds; h: hours; d: days;
OOM: out of memory; OOT: $>3$ d)
\item[{\bf Samplings:}]``Samplings'' records the times of calling the sampling step for the known interpolation methods and the new method. 
\end{itemize}
\end{center}
\end{table}

We clarify  the differences between the new  method for computing
$\LX_{\infty}$  proposed in the current work and the method for computing $\LX_{J}$ proposed in  \citep{Tang2017}.  As what has been mentioned, both $\LX_{J}$ and $\LX_{\infty}$ can be computed by computing
Gr\"obner Bases. Since in practice the likelihood equations are too challenging to tackle by the standard tools for computing  Gr\"obner Bases, an  interpolation (specialization/lifting) method is proposed for computing $\LX_{J}$  in \citep{Tang2017}.
The main idea in this work is  to apply a similar specialization/lifting technique  for computing $\LX_{\infty}$.
But, we need to solve the following two problems.
\begin{itemize}
\item First,  before applying the interpolation method to $\LX_{\infty}$, one should prove that  $\LX_{\infty}$ has the similar specialization properties as $\LX_{J}$ does.
More specifically, in this work, we prove that  the nonproperness set of a  specialized system is equal to the specialized nonproperness set of the system (Theorem \ref{thm:main}) if the specialization is generic, which guarantees the correctness of the key steps in the interpolation method such as computing the degrees and the sampling steps.

\item Second, directly applying the known interpolation methods for computing $\LX_{\infty}$ is still inefficient (see Table \ref{table:literatureOld}).
 In this work, we observe that the generator of the nonproperness set usually has lots of factors, especially linear factors (e.g., see \eqref{eq:symmetric} in Example \ref{ex:symmetric}).  According to \citep{Tang2017}, the computational time of
the interpolation method is decided by the computational time of the sampling step and the times of calling the sampling step in the whole process, where by ``sampling" we mean to compute a Gr\"obner basis after specializing   the parameters in the likelihood equations, and how many times we do the samplings depends on the size (degree) of the objective polynomial we want to interpolate.   So, a crucial idea to save the computational time is that instead of interpolating the whole generator of the nonproperness set, we choose to respectively interpolate its factors, which dramatically improves the efficiency of the interpolation method \citep{Tang2017}. From Table \ref{table:literatureOld},  we can see that  the times of calling the sampling step in   the new  method  are much fewer than that in the old interpolation methods.  For instance, the old method needs over  one million times of  samplings for many models, while it takes only dozens of times for the new method.
\end{itemize}

{\bf Structure of the Paper} The rest of this article is organized as follows.
In Section \ref{sec:thm}, we prove a specialization theorem for the nonproperness sets (Theorem \ref{thm:main}) and we also prove two corollaries (Corollary \ref{coro:Linear} and Corollary \ref{coro:Sample}), which guarantee the correctness of the specialization/interpolation method we propose.
In Section \ref{sec:algorithms}, we first review the standard method (Algorithm \ref{alg:standard}), which is designed for the general parametric algebraic systems, and then,  we present a new method (Algorithm \ref{alg:new} with a list of sub-algorithms) for computing the nonproperness sets of likelihood equations. We also prove the correctness of Algorithm \ref{alg:new} based on the theoretical results proved in the previous section.
In Section \ref{sec:implementation}, we explain the implementation details and compare the efficiency of Algorithm \ref{alg:new} with the known methods (Algorithm \ref{alg:standard} and the old interpolation methods proposed in \citep{Tang2017}).
We end this work with a discussion, see Section \ref{sec:summary}.

\section{Theorem}\label{sec:thm}
In Section \ref{sec:main}, we review the definition of nonproperness set for a general parametric polynomial system, and we present  a specialization property of
the nonproperness sets, see Theorem \ref{thm:main}.
We also present two corollaries (Corollary \ref{coro:Linear} and Corollary \ref{coro:Sample}), which will be used for proving the correctness of the algorithms developed later. The proofs of these results will be given in Sections \ref{subsec:pf-thm-main}--\ref{subsubsec:two-corollaries-pf2}.
For a general overview on the fundamental concepts of computational algebraic geometry, we refer the readers to \citep{CLO2015}.

\subsection{Specialization Properties for
Nonproperness Sets}\label{sec:main}
Let $\vU:=(u_1,\ldots,u_n)$ be a vector of parameters, and let $\vX:=(x_1,\ldots,x_m)$ be a vector of variables.
Define $\pi$ as the \struc{\textit{canonical projection}}:
 ${\mathbb C}^{n}\times {\mathbb
C}^{m}\rightarrow{\mathbb C}^{n}$ such that
for every $(\Vector{u}, \Vector{x})\in {\mathbb C}^{n}\times {\mathbb
C}^{m}$,
$\pi(\Vector{u}, \Vector{x})=\Vector{u}$.
For any  set of polynomials $\vf\subseteq\Q[\vU,\vX]$, we define the \struc{\textit{affine variety}} generated by $\vf$ as
\begin{align}
    \mV(\vf)\;:=\;\{(\vU, \vX)\in {\mathbb C}^{n}\times {\mathbb C}^{m}\mid\vf(\vU,\vX)=0\}.
\end{align}

\begin{definition}[{\bf Nonproperness Set}]\label{def:nddv}\citep{DV2005}
Given  $\vf\subseteq\Q[\vU, \vX]$,  the \struc{\textit{set of nonproperness}} of $\vf$, denoted by $\LX_{\infty}$, is defined  as the set of the $\Vector{u}\in \overline{\pi(\mV(\vf))}$ such that there does not exist a compact neighborhood $U$ of $\Vector{u}$ where
$\pi^{-1}(U)\cap\mV(\vf)$ is  compact.
\end{definition}

Next, we define homomorphisms for representing the specializations we will use in our algorithms.
Let $\vY:=(y_1,\ldots,y_t)$ be another vector of parameters.
Given $\vG:=\{G_1,\ldots, G_n\}\subset\Q[\vY]$,
we say $\varphi:{\mathbb Q}[\Vector{u}, \Vector{x}]\to{\mathbb Q}[\vY,\vX]$ is the \struc{\textit{homomorphism with respect to (w.r.t.) $\vG$}} if for every
$g\in {\mathbb Q}[\Vector{u}, \Vector{x}]$,
\begin{align}\label{eq:def-general-map}
    \varphi(g(\vU,\vX))\; =\; g(G_1(\vY),G_2(\vY),\ldots,G_n(\vY),\Vector{x}).
\end{align}

\begin{example}\label{ex:homomor}
Suppose $G_1=y_1,\;G_2=a_1,\;\ldots,\;G_n=a_{n-1}$, where $a_1,\ldots,a_{n-1}$ are constant numbers in $\Q$. Then, the homomorphism $\varphi$
w.r.t. $\vG$ is exactly the following specialization
\begin{align}\label{eq:def-general-map-eg}
    \varphi(g(\vU,\vX))\; =\; g(y_1,a_1,\ldots,a_{n-1},\Vector{x}).
\end{align}
\end{example}

For any $\vf\subseteq {\mathbb Q}[\Vector{u}, \vX]$, we denote by $\langle \vf \rangle$ the ideal generated by $\vf$ in ${\mathbb Q}[\Vector{u},\vX]$. For any $g\in {\mathbb Q}[\Vector{u}, \vX]$, we denote by $\deg(g, \vX)$ the
degree of
$q$ w.r.t. $\vX$.  For any $S\subseteq\C^n$,
define the ideal generated by $S$ as
\begin{align}
    \mI(S):=\;\{g\in\Q[\vU]\mid g(u^*_1,\ldots,u^*_n)=0\;{\rm for\;all\;}(u^*_1,\ldots,u^*_n)\in S\}.
\end{align}
For any ideal $I$, we denote by $\sqrt{I}$ the radical ideal of $I$.

\begin{theorem}\label{thm:main}
For any $\vf\subseteq\Q[\vU,\vX]$, and for any
$\vG\subset\Q[\vY]$,
suppose  $\varphi:{\mathbb Q}[\Vector{u}, \Vector{x}]\to{\mathbb Q}[\vY,\vX]$ is the homomorphism w.r.t. $\vG$. If
\begin{enumerate}
    \item [(1)] for any $g\in\langle\vf\rangle$, $\deg(g,\vX)=\deg(\varphi(g),\vX)$, and
    \item [(2)] $\varphi$ is onto,
\end{enumerate}
then
\begin{enumerate}
    \item [(i)] $\varphi(\mI(\mV(\vf)_{\infty}))$ is an ideal, and
    \item [(ii)] 
    $\sqrt{\varphi(\mI(\mV(\vf)_{\infty}))}\;=\;\mI(\mV(\varphi(\vf))_{\infty})$.
\end{enumerate}
\end{theorem}
Theorem \ref{thm:main} will play a key role in the proof of the main algorithm (Algorithm \ref{alg:new}) presented in  Section \ref{sec:algorithms}.
For instance,
we can view a general linear transformation to the coordinates $u_1, \ldots, u_m$ as a homomorphism, and we will get a corollary by applying Theorem \ref{thm:main} to this homomorphism, see Corollary \ref{coro:Linear}.
Similarly, if $a_1,\ldots,a_{n-1}$ in Example \ref{ex:homomor} are generic numbers,  then Theorem \ref{thm:main} holds for the homomorphism
$\varphi$ defined in \eqref{eq:def-general-map-eg} (see Corollary \ref{coro:Sample}).
We present the precise statements  of these corollaries in as follows.


Define $\vV:=(v_1,\ldots,v_n)$.
For any vector $\vA:=(a_1,\ldots,a_{n-1})\in\Q^{n-1}$,
we define a homomorphism $\tau_{\vA}:\;{\mathbb Q}[\Vector{u},\vX]\to{\mathbb Q}[\vV,\vX]$ such that
\begin{align}\label{eq:def-linear-map}
    \tau_{\vA}(g(\vU,\vX))\;=\;g(v_1,v_2+a_1v_1,\ldots,v_n+a_{n-1}v_1,\vX)\; \text{for any}\; g\in \Q[\vU,\vX].
\end{align}

\begin{corollary}\label{coro:Linear}
Let $\vf\subseteq\Q[\vU,\vX]$.
If $\mI(\mV(\vf)_{\infty})=\langle w\rangle$, where $w\in\Q[\vU]$,
then for any $\vA\in\Q^{n-1}$, we have
\begin{align*}
    \mI(\mV(\tau_{\vA}(\vf))_{\infty})\;=\;\langle\tau_{\vA}(w)\rangle.
\end{align*}
\end{corollary}


For any point
$\Vector{b}:=(b_1,\ldots,b_{n-1})\in {\mathbb Q}^{n-1}$,
we define the homomorphism $\sigma_{\vB}:\Q[\vU,\vX]\rightarrow\Q[u_1,\vX]$ such that
\begin{align}\label{eq:def-sample-map}
    \sigma_{\vB}(g(\vU,\vX))\;=\;g(u_1,b_1,\ldots,b_{n-1},\vX)\; \text{for any}\; g\in \Q[\vU,\vX].
\end{align}

\begin{corollary}\label{coro:Sample}
Let $\vf\subseteq\Q[\vU,\vX]$.
If $\mI(\mV(\vf)_{\infty})=\langle w\rangle$, where $w\in\Q[\vU]$,
then
there exists a Zariski dense subset $\Theta\subseteq\Q^{n-1}$ and an affine variety $V\subset\C^{n-1}$ such that for any $\vB\in\Theta\setminus V$,
\begin{align*}
    \mI(\mV(\sigma_{\vB}(\vf))_{\infty})\;=\;\langle\sigma_{\vB}(w)\rangle.
\end{align*}
\end{corollary}

 Later in Section \ref{sec:algorithms}, we will use these corollaries to compute the degrees of the generator polynomial of the nonproperness set, and to do the sampling steps in the interpolation method.

 \subsection{Proof of Theorem \ref{thm:main}}\label{subsec:pf-thm-main}
The lemmas for proving Theorem \ref{thm:main} are organized as follows.
We first prove two basic properties
of the homomorphism defined in \eqref{eq:def-general-map}:
Lemma \ref{lemma:homogenization} and Lemma \ref{lemma:homo-ideal}. Then, we review a known result
 \citep[Lemma 2]{DV2005}, which says that
a nonproperness set is algebraically  closed, see Lemma \ref{lemma:LR2005}.
Using these lemmas, we can prove the last key lemma:  Lemma \ref{lemma:main}, and Theorem \ref{thm:main} easily follows from Lemma \ref{lemma:main}.

Let $\mK$ be a polynomial ring
(for instance, $\mK$ can be $\Q[\vU]$ or $\Q[\vY]$ in the rest of the paper).
For any $g\in\mK[\vX]$, we say that $g$ is an \textcolor{blue}{$\vX$-homogeneous polynomial} if
$g$ is homogeneous w.r.t. $\vX$, i.e.,
\begin{align}
    g\;=\;\sum_{|\alpha|=d}C_{\alpha}\cdot\vX^{\alpha},
\end{align}
where $d:=\deg(g, \vX)$, $C_{\alpha}\in\mK$,  $\alpha:=(\alpha_1,\ldots,\alpha_m)\in\Z_{\ge0}^{m}$, $\vX^{\alpha}:=\prod_{i=1}^mx_i^{\alpha_i}$ and $|\alpha|:=\alpha_1+\cdots+\alpha_m$.
Let $T$ be a new variable.
For any $g\in \mK[\vX]$, define the \textcolor{blue}{$(\vX,T)$-homogenization of $g$}, denoted by
$g^{h}$,  as the $(\vX,T)$-homogeneous polynomial in $\mK[\vX,T]$ such that $\deg(g^{h},(\vX,T))=\deg(g,\vX)$ and $g^{h}(\vX,1)=g$. For any ideal $I\subseteq\mK[\vX]$, let $I^h$ be the ideal generated by the $(\vX,T)$-homogenizations of polynomials in $I$.
The above definitions can be seen in \citep[Proposititon 9, page 373]{CLO2015}.

Given $G_i\in\Q[\vY]\;(1\le i\le n)$.
We extend the homomorphsim $\varphi:\;\Q[\vU,\vX]\rightarrow\Q[\vY,\vX]$ w.r.t $G_i$  to  $\varphi:\;\Q[\vU,\vX,T]\rightarrow\Q[\vY,\vX,T]$ such that
\begin{align}\label{eq:def-general-map-extend}
    \varphi(g(\vU,\vX,T))\; =\; g(G_1(\vY),\ldots,G_n(\vY),\Vector{x},T).
\end{align}
Note that if $\varphi$ defined in \eqref{eq:def-general-map} is onto,
then the extension in \eqref{eq:def-general-map-extend} is still onto.

\begin{lemma}\label{lemma:homogenization}
Given $G_i\in\Q[\vY]\;(1\le i\le n)$.
For any $q\in\Q[\vU,\vX]$,
if the homomorphism $\varphi:\;\Q[\vU,\vX]\rightarrow\Q[\vY,\vX]$ w.r.t $G_i$ satisfies that
$\deg(q,\vX)=\deg(\varphi(q),\vX)$,
then $\varphi(q)^{h}=\varphi(q^{h})$.
\end{lemma}
\begin{proof}
Suppose that $\deg(q,\vX)=d$.
Write $q$ as the sum of its homogeneous components w.r.t. $\vX$:
\begin{align}\label{eq:poly-homogeneous-component}
    q\;=\;\sum_{i=0}^{d}\left(\sum_{|\alpha|=i}C_{\alpha}\cdot\vX^{\alpha}\right),\;{\rm where}\;\alpha\in\Z_{\ge0}^m\;{\rm and}\;C_{\alpha}\in\Q[\vU].
\end{align}
So,
\begin{align}\label{eq:poly^h-homogeneous-component}
    q^h\;=\;\sum_{i=0}^{d}\left(\sum_{|\alpha|=i}C_{\alpha}\cdot\vX^{\alpha}\cdot T^{d-i}\right).
\end{align}
 Then, by \eqref{eq:def-general-map-extend}, \eqref{eq:poly-homogeneous-component} and \eqref{eq:poly^h-homogeneous-component}, we have
\begin{align}
    \varphi(q)\;&=\;\sum_{i=0}^{d}\left(\sum_{|\alpha|=i}\varphi(C_{\alpha})\cdot\vX^{\alpha}\right),\;{\rm and}\label{eq:sigma-q}\\
    \varphi(q^h)\;&=\;\sum_{i=0}^{d}\left(\sum_{|\alpha|=i}\varphi(C_{\alpha})\cdot\vX^{\alpha}\cdot T^{d-i}\right),\;{\rm respectively}.\label{eq:sigma-q^h}
\end{align}
If $\deg(\varphi(q),\vX)=\deg(q,\vX)=d$, then by \eqref{eq:sigma-q},
\begin{align}
    \varphi(q)^h\;=\;\sum_{i=0}^{d}\left(\sum_{|\alpha|=i}\varphi(C_{\alpha})\cdot\vX^{\alpha}\cdot T^{d-i}\right).\label{eq:(sigma-q)^h}
\end{align}
By \eqref{eq:sigma-q^h} and \eqref{eq:(sigma-q)^h}, we complete the proof.
\end{proof}

\begin{lemma}\label{lemma:homo-ideal}
Given $G_i\in\Q[\vY]\;(1\le i\le n)$.
If the homomorphism $\varphi:\;\Q[\vU,\vX]\rightarrow\Q[\vY,\vX]$ w.r.t $G_i$  is onto, then for any ideal $\langle\vf\rangle\subseteq\Q[\vU,\vX]$,  we have
\begin{align}\label{eq:homo-ideal}
    \varphi(\langle \vf \rangle)\;=\;\langle\varphi(\vf)\rangle.
\end{align}
\end{lemma}
\begin{proof}
First, we prove that
$\varphi(\langle \vf\rangle)\subseteq\langle\varphi(\vf)\rangle$.
For any $q\in\varphi(\langle \vf\rangle)$,
there exists $g\in\langle \vf\rangle$ such that
$q=\varphi(g)$.
Note that $g\in\langle \vf\rangle$.
So,
\begin{align*}
    g\;=\;\sum_{g_i\in \vf}h_i\cdot g_i,\;{\rm where}\;h_i\in\Q[\vU,\vX],
\end{align*}
and hence,
\begin{align*}
  q=  \varphi(g)\;=\;\sum_{g_i\in\vf}\varphi(h_i)\cdot\varphi(g_i)\in \langle\varphi(\vf)\rangle.
\end{align*}
Second, we prove that
$\langle\varphi(\vf)\rangle\subseteq\varphi(\langle\vf\rangle)$.
For any $q\in\langle\varphi(\vf)\rangle$, we have
\begin{align*}
    q\;=\;\sum_{g_i\in \vf}H_i\cdot\varphi(g_i),\;{\rm where}\;H_i\in\Q[\vY,\vX].
\end{align*}
Since $\varphi$ in \eqref{eq:def-general-map} is onto, for every $H_i\in\Q[\vY,\vX]$, there exists $h_i\in\Q[\vU,\vX]$ such that
$H_i\;=\;\varphi(h_i)$.
So, we have
\begin{align*}
    q\;=\;\sum_{g_i\in \vf}\varphi(h_i)\cdot\varphi(g_i)
    \;=\;\varphi(\sum_{g_i\in \vf}h_i\cdot g_i)\in\varphi(\langle \vf\rangle).
\end{align*}
\end{proof}

Throughout the rest of the paper, we denote
by $\bP$ the projective closure of $\C$.
\begin{lemma}\cite[Lemma 2]{DV2005}\label{lemma:LR2005}
Let $\vf\subseteq\Q[\vU,\vX]$. Then,
\begin{align}
    \mV(\vf)_{\infty}\;=\;
    \pi^{(n)}(\mV_{\bP}(\langle\vf\rangle^h)\cap\mH^{n\times m}),
\end{align}
where
\begin{enumerate}
    \item [$(1)$] $\mV_{\bP}(\langle\vf\rangle^h):=\{(\vU^*,\vX^*,T^*)\in\C^n\times\bP^{m}\mid g(\vU^*,\vX^*,T^*)=0\;{\rm for\;any}\;g\in\langle\vf\rangle^h\}$,
    \item [$(2)$] $\mH^{n\times m}$ denotes the hyperplane at infinity in $\C^{n}\times\bP^{m}$, i.e., $\mH^{n\times m}=(\C^{n}\times\bP^{m})\setminus(\C^{n}\times\C^{m})$, and
    \item [$(3)$] $\pi^{(n)}$ denotes the canonical projection from $\C^{n}\times\bP^{m}$ to $\C^{n}$.
\end{enumerate}
\end{lemma}

\begin{lemma}\label{lemma:main}
Let $\vf\subseteq\Q[\vU,\vX]$.
Given $G_i\in\Q[\vY]\;(1\le i\le n)$.
If the homomorphism $\varphi:\;\Q[\vU,\vX]\rightarrow\Q[\vY,\vX]$ w.r.t $G_i$ satisfies
\begin{enumerate}
    \item[$(1)$] for any $q\in\langle\vf\rangle$, $\deg(q,\vX)=\deg(\varphi(q),\vX)$, and
    \item[$(2)$] $\varphi$ is onto,
\end{enumerate}
then $\mV(\varphi(\mI(\mV(\vf)_{\infty})))=\mV(\varphi(\vf))_{\infty}.$
\end{lemma}
\begin{proof}
By the hypotheses (1) and (2), we have two claims below.

{\bf (Claim 1).}\; $\varphi(\langle\vf\rangle^h)=\langle\varphi(\vf)\rangle^h$.

{\bf (Claim 2).}\; For any $\vY^*\in\mV(\varphi(\vf))_{\infty}$, $(G_1(\vY^*),\ldots,G_n(\vY^*))\in\mV(\vf)_{\infty}$.

We prove the conclusion based on the two claims (the proofs of the claims will be done later).
Assume that $\mI(\mV(\vf)_{\infty})=\langle W\rangle$, where $W\subseteq\Q[\vU]$.
Then,
\begin{align}\label{eq:def-subset-w-infinity}
    \mV(\vf)_{\infty}\;=\;\mV(W).
\end{align}
So, by Lemma \ref{lemma:homo-ideal}, we only need to prove
\begin{align*}
    \mV(\langle\varphi(W)\rangle)\;=\;\mV(\varphi(\vf))_{\infty},\;\; {\rm i.e.}, \;\; \mV(\varphi(W))\;=\;\mV(\varphi(\vf))_{\infty}.
\end{align*}
First, we prove that $\mV(\varphi(\vf))_{\infty}\subseteq\mV(\varphi(W))$.
For any $\vY^*\in\mV(\varphi(\vf))_{\infty}$,
by Claim 2, we have $(G_1(\vY^*),\ldots,G_n(\vY^*))\in\mV(\vf)_{\infty}$.
Then, by \eqref{eq:def-subset-w-infinity},
\begin{align*}
    W(G_1(\vY^*),\ldots,G_n(\vY^*))\;=\;0.
\end{align*}
So, by the definition of $\varphi$ (see \eqref{eq:def-general-map}), we have $\varphi(W)(\vY^*)=0$, i.e., $\vY^*\in\mV(\varphi(W))$.
Second, we prove that $\mV(\varphi(W))\subseteq\mV(\varphi(\vf))_{\infty}$.
For any $\vY^*\in\mV(\varphi(W))$, we have $\varphi(W)(\vY^*)=0$.
By \eqref{eq:def-general-map},
$W(G_1(\vY^*),\ldots,G_n(\vY^*))=0$.
Define
\begin{align}\label{eq:def-star-again}
    \vU^*:=\;(G_1(\vY^*),\ldots,G_n(\vY^*)).
\end{align}
Then, $W(\vU^*)=0$.
So, by \eqref{eq:def-subset-w-infinity},
we have $\vU^*\in\mV(\vf)_{\infty}$.
By Lemma \ref{lemma:LR2005},
there exists $(\vX^*,T^*)$ such that
\begin{align}\label{eq:InCap}
    (\vU^*,\vX^*,T^*)\;\in\;\mV_{\bP}(\langle\vf\rangle^h)\cap\mH^{n\times m}
\end{align}
and $\pi^{(n)}(\vU^*,\vX^*,T^*)=\vU^*$.
By Lemma \ref{lemma:LR2005}, in the rest of the proof, we only need to show that
\begin{align*}
    (\vY^*,\vX^*,T^*)\;\in\;\mV_{\bP}(\langle\varphi(\vf)\rangle^h)\cap\mH^{t\times m}.
\end{align*}
Note that in \eqref{eq:InCap}, $(\vX^{*},T^*)\in\bP^m\setminus\C^m$.
Note also that
$\mH^{t\times m}=(\C^{t}\times\bP^{m})\setminus(\C^{t}\times\C^{m})$.
So, $(\vY^*,\vX^*,T^*)\in\mH^{t\times m}$.
Next, we prove that
$(\vY^*,\vX^*,T^*)\in\mV_{\bP}(\langle\varphi(\vf)\rangle^h)$.
For any $g\in\langle\varphi(\vf)\rangle^h$,
by Claim 1,
we have $g\in\varphi(\langle\vf\rangle^h)$.
So, there exists $q\in\langle\vf\rangle^h$ such that
\begin{align}\label{eq:representation}
    g\;=\;\varphi(q).
\end{align}
Note that  by \eqref{eq:InCap},
$q(\vU^*,\vX^*,T^*)=0$.
Then, by \eqref{eq:def-general-map-extend} and by  \eqref{eq:def-star-again}, we have
\begin{align}\label{eq:by-two-defs}
    \varphi(q)(\vY^*,\vX^*,T^*)\;=\;0.
\end{align}
By \eqref{eq:representation} and \eqref{eq:by-two-defs}, we have $g(\vY^*,\vX^*,T^*)=0$, i.e.,
$(\vY^*,\vX^*,T^*)\in\mV_{\bP}(\langle\varphi(\vf)\rangle^h)$. We complete the proof.

{\bf Proof of Claim 1.}
First, we prove that $\varphi(\langle\vf\rangle^h)\subseteq\langle\varphi(\vf)\rangle^h$.
For any $g\in\varphi(\langle\vf\rangle^h)$, there exists $q\in\langle\vf\rangle^h$ such that
\begin{align}\label{eq:relation-h}
    g\;=\;\varphi(q).
\end{align}
Note that $q\in\langle\vf\rangle^h$, i.e.,
\begin{align}\label{eq:repre}
    q\;=\;\sum_{q_i\in\langle\vf\rangle}r_i\cdot q_i^h,\;{\rm where}\;r_i\in\Q[\vU,\vX,T].
\end{align}
So, by \eqref{eq:repre} and by Lemma \ref{lemma:homogenization}, we have
\begin{align}\label{eq:apply-lemma:homogenization}
    \varphi(q)\;=\;\sum_{q_i\in\langle\vf\rangle}\varphi(r_i)\cdot\varphi(q_i^h)\;=\;\sum_{q_i\in\langle\vf\rangle}\varphi(r_i)\cdot\varphi(q_i)^h.
\end{align}
Note that in \eqref{eq:apply-lemma:homogenization}, $\varphi(q_i)\in\varphi(\langle\vf\rangle)$.
Hence, by Lemma \ref{lemma:homo-ideal}, we have
$\varphi(q_i)\in\langle\varphi(\vf)\rangle$.
So, by \eqref{eq:relation-h} and \eqref{eq:apply-lemma:homogenization}, $g\in\langle\varphi(\vf)\rangle^h$.
Second, we prove that $\langle\varphi(\vf)\rangle^h\subseteq\varphi(\langle\vf\rangle^h)$.
For any $g\in\langle\varphi(\vf)\rangle^h$, we have
\begin{align}\label{eq:second-prove}
    g\;=\;\sum_{g_i\in\langle\varphi(\vf)\rangle}r_i\cdot g_i^h,\;{\rm where}\;r_i\in\Q[\vY,\vX,T].
\end{align}
Note that in \eqref{eq:second-prove}, $g_i\in\langle\varphi(\vf)\rangle$.
So, by Lemma \ref{lemma:homo-ideal},
$g_i\in\varphi(\langle\vf\rangle)$.
Hence, there exists $q_i\in\langle\vf\rangle$ such that $g_i=\varphi(q_i)$.
Thus, by \eqref{eq:second-prove}, we have
\begin{align}\label{eq:by-eq:second-prove}
    g\;=\;\sum_{q_i\in\langle\vf\rangle}r_i\cdot\varphi(q_i)^h.
\end{align}
Since $\varphi$ in \eqref{eq:def-general-map-extend} is onto, for every $r_i\in\Q[\vY,\vX,T]$,
there exists $\ell_i\in\Q[\vU,\vX,T]$ such that $r_i=\varphi(\ell_i)$.
So, by \eqref{eq:by-eq:second-prove},
\begin{align*}
    g\;=\;\sum_{q_i\in\langle\vf\rangle}\varphi(\ell_i)\cdot\varphi(q_i)^h.
\end{align*}
Thus, by Lemma \ref{lemma:homogenization}, we have
\begin{align*}
    g\;=\;\sum_{q_i\in\langle\vf\rangle}\varphi(\ell_i)\cdot\varphi(q_i^h)
    \;=\;\varphi(\sum_{q_i\in\langle\vf\rangle}\ell_i\cdot q_i^h).
\end{align*}
Therefore, $g\in\varphi(\langle\vf\rangle^h)$.

{\bf Proof of Claim 2.}
By Lemma \ref{lemma:LR2005}, we have
\begin{align}\label{eq:LR2005-main}
    \mV(\varphi(\vf))_{\infty}\;=\;\pi^{(t)}(\mV_{\bP}(\langle\varphi(\vf)\rangle^h)\cap\mH^{t\times m}).
\end{align}
For any $\vY^{*}\in\mV(\varphi(\vf))_{\infty}$,
by \eqref{eq:LR2005-main},
there exists $(\vX^*,T^*)$ such that
\begin{align}\label{eq:in-cap}
    (\vY^*,\vX^*,T^*)\;\in\;\mV_{\bP}(\langle\varphi(\vf)\rangle^h)\cap\mH^{t\times m}
\end{align}
and $\pi^{(t)}(\vY^*,\vX^*,T^*)=\vY^*$.
Define
\begin{align}\label{eq:define-star}
    \vU^*:=\;(G_1(\vY^*),\ldots,G_n(\vY^*)).
\end{align}
By Lemma \ref{lemma:LR2005}, we only need to prove that
\begin{align*}
    (\vU^*,\vX^*,T^*)\;\in\;\mV_{\bP}(\langle\vf\rangle^h)\cap\mH^{n\times m}.
\end{align*}
Note that by \eqref{eq:in-cap}, $(\vX^*,T^*)\in\bP^m\setminus\C^m$.
So, $(\vU^*,\vX^*,T^*)\in\mH^{n\times m}$.
Below, we prove that  $(\vU^*,\vX^*,T^*)\in\mV_{\bP}(\langle\vf\rangle^h)$.
For any $g\in\langle\vf\rangle^h$, $\varphi(g)\in\varphi(\langle\vf\rangle^h)$.
So, by Claim 1, we have
$\varphi(g)\in\langle\varphi(\vf)\rangle^h$.
Hence, by \eqref{eq:in-cap},
$\varphi(g)(\vY^*,\vX^*,T^*)=0$.
Thus, by \eqref{eq:def-general-map-extend}) and \eqref{eq:define-star}, we have
\begin{align*}
    g(\vU^*,\vX^*,T^*)\;=\;0,\;\;{\rm i.e.},\;\; (\vU^*,\vX^*,T^*)\in\mV_{\bP}(\langle\vf\rangle^h).
\end{align*}
We complete the proof.
\end{proof}

\noindent\textbf{Proof of Theorem \ref{thm:main}.}
\begin{proof}
By Lemma \ref{lemma:main},
\begin{align*}
    \mV(\varphi(\mI(\mV(\vf)_{\infty})))\;=\;\mV(\varphi(\vf))_{\infty}.
\end{align*}
So,
\begin{align}\label{eq:apply-op-ideal}
    \mI(\mV(\varphi(\mI(\mV(\vf)_{\infty}))))\;=\;\mI(\mV(\varphi(\vf))_{\infty}).
\end{align}
By Lemma \ref{lemma:homo-ideal}, we have
$\varphi(\mI(\mV(\vf)_{\infty}))$ is an ideal.
Then, by \cite[page 176, Theorem 6]{CLO2015}, we have
\begin{align}\label{eq:apply-strong-Nullstellensatz}
\mI(\mV(\varphi(\mI(\mV(\vf)_{\infty}))))\;=\;\sqrt{\varphi(\mI(\mV(\vf)_{\infty}))}.
\end{align}
By \eqref{eq:apply-op-ideal} and \eqref{eq:apply-strong-Nullstellensatz}, we have
\begin{align*}
\sqrt{\varphi(\mI(\mV(\vf)_{\infty}))}\;=\;\mI(\mV(\varphi(\vf))_{\infty}).
\end{align*}
\end{proof}

\subsection{Proof of Corollary \ref{coro:Linear}  }\label{subsubsec:two-corollaries-pf1} Since
Corollary \ref{coro:Linear} claims that its conclusion holds for  $\tau_{\vA}$
defined in \eqref{eq:def-linear-map}, no matter which $\vA$ we choose from ${\mathbb Q}^{n-1}$, in this proof we simply denote $\tau_{\vA}$ by $\tau$ by assuming that
$\vA$ is fixed.
First, we prove in
Lemma \ref{lemma:linear-deg-equal} and Lemma \ref{lemma:linear-bijection} that
the homomorphism $\tau$  satisfies the hypotheses  (i) and (ii) of Theorem \ref{thm:main}. After that, we show that in
Lemma \ref{lemma:linear-irreducible} the homomorphism $\tau$ keeps the irreducibility. Finally, using these results, we apply   Theorem \ref{thm:main} to the homomorphism $\tau$, and we get Corollary \ref{coro:Linear}.

\begin{lemma}\label{lemma:linear-deg-equal}
For any $q\in\Q[\vU,\vX]$ and for the homomorphism $\tau$ defined in \eqref{eq:def-linear-map}, we have the following statements.
\begin{itemize}
    \item [$(\romannumeral1)$] If $q\in\Q[\vU,\vX]\setminus\{0\}$, then $\tau(q)\in\Q[\vV,\vX]\setminus\{0\}$.
    \item [$(\romannumeral2)$] $\deg(q,\vX)=\deg(\tau(q),\vX)$.
\end{itemize}
\end{lemma}
\begin{proof}
$(\romannumeral1)$
Notice that by the definition of $\tau$, 
for any $q\in\Q[\vU,\vX]$,
we have
\begin{align*}
    \tau(q)(u_1,u_2-a_1u_1,\ldots,u_n-a_{n-1}u_1,\vX)\;\equiv\;q.
\end{align*}
So, if $q$ is not the zero polynomial, then $\tau(q)$ is not the zero polynomial.

$(\romannumeral2)$ 
If $q\equiv 0$, then $\tau(q)\equiv 0$, and hence, the conclusion holds.
If $\deg(q,\vX)=d$, where $d\in\Z_{\ge 0}$, then we suppose that
\begin{align*}
    q\;=\;\sum_{\alpha}C_{\alpha}\cdot\vX^{\alpha},\;{\rm where\;}\alpha\in\Z_{\ge0}^m,\;\max\limits_{\alpha}\{|\alpha|\}=d\;{\rm and\;}C_{\alpha}\in\Q[\vU]\setminus\{0\}.
\end{align*}
Then, by the definition of the homomorphism $\tau$, 
we have
\begin{align}\label{eq:apply-def-linear-map}
    \tau(q)\;=\;\sum_{\alpha}\tau(C_{\alpha})\cdot\vX^{\alpha}.
\end{align}
Note that $\tau(C_{\alpha})\in\Q[\vV]$.
Then, since $C_{\alpha}\in\Q[\vU]\setminus\{0\}$, by the statement $(\romannumeral1)$,
\begin{align}\label{eq:ne0}
    \tau(C_{\alpha})\;\in\;\Q[\vV]\setminus\{0\}
\end{align}
Note that $\max\limits_{\alpha}\{|\alpha|\}=d$.
So, by \eqref{eq:apply-def-linear-map} and \eqref{eq:ne0}, we have $\deg(\tau(q),\vX)=d$.
\end{proof}

\begin{lemma}\label{lemma:linear-bijection}
The homomorphism $\tau$ defined  in \eqref{eq:def-linear-map} is a bijection.
\end{lemma}
\begin{proof}
First, we prove that $\tau$ is onto.
Note that for any polynomial $g\in\Q[\vV,\vX]$,
\begin{align*}
    \tau(g(u_1,u_2-a_1u_1,\ldots,u_n-a_{n-1}u_1,\vX))\;=\;g.
\end{align*}
Second, we prove that $\tau$ is injective.
For any polynomials $q_1,q_2\in\Q[\vU,\vX]$ such that $q_1\ne q_2$,
by Lemma \ref{lemma:linear-deg-equal} $(\romannumeral1)$, $\tau(q_1-q_2)\neq 0$, and hence, $\tau(q_1)\neq \tau(q_2)$.
\end{proof}

\begin{lemma}\label{lemma:linear-irreducible}
For any $q\in\Q[\vU]$ and for the homomorphism $\tau$ defined in \eqref{eq:def-linear-map}, we have 
\begin{enumerate}
\item[$(\romannumeral1)$]  $\deg(\tau(q), \vV)\le\deg(q, \vU)$, and
\item[$(\romannumeral2)$]  if $q$ is an irreducible polynomial in $\Q[\vU]$,
then $\tau(q)$ is an irreducible polynomial in $\Q[\vV]$.
\end{enumerate}
\end{lemma}
\begin{proof}
$(\romannumeral1)$ 
Suppose that
\begin{align*}
    q\;=\;\sum_{\alpha=(\alpha_1,\ldots,\alpha_n)\in\Z^{n}_{\ge0}}C_{\alpha}\cdot u_{1}^{\alpha_1}\cdots u_{n}^{\alpha_n},\;{\rm where}
    \;C_{\alpha}\in\Q.
\end{align*}
Then, by the definition of $\tau$,
\begin{align*}
    \tau(q)\;=\;\sum_{\alpha}C_{\alpha}\cdot v_1^{\alpha_1}(v_2+a_1v_1)^{\alpha_2}\cdots (v_n+a_{n-1}v_1)^{\alpha_n}.
\end{align*}
So, we have $\deg(\tau(q), \vV)\le\deg(q, \vU)$.

$(\romannumeral2)$
If $\tau(q)$ is not an irreducible polynomial in $\Q[\vV]$, then there exist two
polynomials $\xi_1,\xi_2\in\Q[\vV]$ such that
\begin{align}\label{eq:tau-irrducible}
    \tau(q)\;=\;\xi_{1}\xi_{2},\;\text{where}\; \deg(\xi_1,\vV)\ge1,\;\text{and}\; \deg(\xi_2,\vV)\ge1.
\end{align}
By Lemma \ref{lemma:linear-bijection}, there exist  $q_1, q_2\in\Q[\vU]$ such that $\xi_1=\tau(q_1)$ and $\xi_2=\tau(q_2)$.
By \eqref{eq:tau-irrducible}, $\tau(q)=\tau(q_1q_2)$.
So, by Lemma \ref{lemma:linear-bijection},
\begin{align}\label{eq:irrducible}
    q\;=\;q_{1}q_{2}.
\end{align}
Note that by $(\romannumeral1)$, $\deg(q_1, \vU)\ge\deg(\xi_1, \vV)\ge1$ and $\deg(q_2, \vU)\ge\deg(\xi_2, \vV)\ge1$.
Therefore, \eqref{eq:irrducible} is a contradiction to the hypothesis that $q$ is an irreducible polynomial in $\Q[\vU]$.
\end{proof}

\begin{lemma}\label{lemma:linear-radical}
Let $\vf\subseteq\Q[\vU,\vX]$.
If $\mI(\mV(\vf)_{\infty})=\langle w\rangle$,
where $w\in\Q[\vU]$,
then the ideal $\langle\tau(w)\rangle$ is radical.
\end{lemma}
\begin{proof}
By \cite[page 180, Proposititon 9]{CLO2015},
we only need to prove that $\tau(w)$ is square-free in $\Q[\vV]$.
Note that $\mathcal{I}(\LX_{\infty})=\langle w\rangle$ is radical.
So, by \cite[page 180, Proposititon 9]{CLO2015},
$w$ is square-free in $\Q[\vU]$.
Suppose that
\begin{align}
    w\;=\;c\prod_{k=1}^{r}w_{k},
\end{align}
where $c\in\Q$, $w_k$ is an irreducible polynomial in $\Q[\vU]$ and $w_{k_1}\ne w_{k_2}$ for any $k_1\ne k_2$.
Then,
\begin{align}\label{eq:irreducible-decomposition}
    \tau(w)\;=\;c\prod_{k=1}^{r}\tau(w_{k}).
\end{align}
Since $w_k$ is an irreducible polynomial in $\Q[\vU]$, by
Lemma \ref{lemma:linear-irreducible} $(\romannumeral2)$, $\tau(w_k)$ is an
irreducible polynomial in $\Q[\vV]$.
Note that for any $k_1\ne k_2$, $w_{k_1}\ne w_{k_2}$.
Then, by Lemma \ref{lemma:linear-bijection}, $\tau(w_{k_1})\ne \tau(w_{k_2})$.
Therefore, by \eqref{eq:irreducible-decomposition}, $\tau(w)$ is square-free.
\end{proof}

\noindent\textbf{Proof of Corollary \ref{coro:Linear}.}
\begin{proof}
By Theorem \ref{thm:main}, Lemma \ref{lemma:linear-deg-equal} $(\romannumeral2)$ and Lemma \ref{lemma:linear-bijection},
\begin{align*}
    \mI(\mV(\tau(\vf))_{\infty})\;=\;\sqrt{\tau(\langle w\rangle)}.
\end{align*}
By Lemma \ref{lemma:homo-ideal} and Lemma \ref{lemma:linear-bijection}, $\tau(\langle w\rangle)=\langle \tau(w)\rangle$.
So,
\begin{align*}
    \sqrt{\tau(\langle w\rangle)}\;=\;\sqrt{\langle \tau(w)\rangle}.
\end{align*}
By Lemma \ref{lemma:linear-radical},
the ideal $\langle\tau(w)\rangle$ is radical.
So,
\begin{align*}
    \sqrt{\langle\tau(w)\rangle}\;=\;\langle\tau(w)\rangle.
\end{align*}
We complete the proof.
\end{proof}

\subsection{Proof of Corollary \ref{coro:Sample} }\label{subsubsec:two-corollaries-pf2}
We first present Theorem \ref{thm:Hilbert}, which directly follows from Hilbert's Irreducibility Theorem. Notice that the homomorphism $\sigma_{\vB}$
defined in \eqref{eq:def-sample-map} naturally satisfies the hypothesis (ii) in Theorem \ref{thm:main}, see Lemma \ref{lemma:map-sample-onto}. So, we spend our main effort on showing that  the homomorphism $\sigma_{\vB}$ also
satisfies the hypothesis (i) in Theorem \ref{thm:main} if $\vB$ is ``generic".

\begin{theorem}
\label{thm:Hilbert}
Let $q\in{\mathbb Q}[\vU]$.
Let $\sigma_{\vB}$ be the homomorphism defined in \eqref{eq:def-sample-map}.
If $q=p\Pi_{k=1}^rq_k^{m_k}$, where $p\in {\mathbb Q}[u_2, \ldots, u_n]$, every $q_k$ is an irreducible polynomial in ${\mathbb Q}[\vU]$ with $\deg(q_k, u_1)>0$, and $q_{k_1}\neq q_{k_2}$ for any $k_1\neq k_2$,
then
there exists a Zariski dense subset $\Theta\subseteq\Q^{n-1}$ and an affine variety $V\subset\C^{n-1}$ such that
for any $\vB\in\Theta\setminus V$, we have $\sigma_{\vB}(q)=\sigma_{\vB}(p)\Pi_{k=1}^r\sigma_{\vB}(q_k)^{m_k}$, where $\sigma_{\vB}(q_k)$ is an irreducible polynomial in ${\mathbb Q}[u_1]$, $\sigma_{\vB}(q_{k_1})\neq\sigma_{\vB}(q_{k_2})$ for any $k_1\neq k_2$, and $deg(\sigma_{\vB}(q_k), u_1)=\deg(q_k,u_1)$.
\end{theorem}

\begin{proof}
Given that $q_1,\ldots,q_r$ are irreducible polynomials, by Hilbert’s Irreducibility Theorem (see e.g., \citep[page 219]{FJ2008}),
there exists a Zariski dense subset $\Theta\subseteq\Q^{n-1}$ such that for any $\vB\in\Theta$,
$\sigma_{\vB}(q_1),\ldots,\sigma_{\vB}(q_r)$ are irreducible polynomials in $\Q[u_1]$.
Let
\begin{align*}
    \zeta_{i,k}:=\;\{\vB\in\C^{n-1}\mid \sigma_{\vB}(q_i-q_k)=0\}\;{\rm and}\;
    \chi_{i}:=\;\{\vB\in\C^{n-1}\mid\sigma_{\vB}(\lc_{\prec{u_1}}(q_i))=0\},
\end{align*}
where $\lc_{\prec{u_1}}(q_i)$ denotes the leading coefficient of $q_i$ w.r.t. $u_1$.
Obviously, $\zeta_{i,k}$ or $\chi_{i}$ is an affine variety, which is not equal to $\C^{n-1}$.
So, let
$$V:=\;\bigcup_{i=1}^{r}\bigcup_{k=1}^{i-1}\zeta_{i,k}\cup\bigcup_{i=1}^{r}\chi_i,$$
and we complete the proof.
\end{proof}


\begin{lemma}\label{lemma:map-sample-onto}
The homomorphism $\sigma_{\vB}$ in \eqref{eq:def-sample-map} is onto for any $\vB\in {\mathbb Q}^{n-1}$.
\end{lemma}
\begin{proof}
For any polynomial $g\in\Q[u_1]$, note that $\sigma_{\vB}(g)=g$.
\end{proof}

\begin{lemma}\label{lemma:sample-radical}
Let $\vf\subseteq\Q[\vU,\vX]$.
If $\mI(\mV(\vf)_{\infty})=\langle w\rangle$, where $w\in\Q[\vU]$,
then
there exists a Zariski dense subset $\Theta\subseteq\Q^{n-1}$ and an affine variety $V\subset\C^{n-1}$ such that for any $\vB\in\Theta\setminus V$,
the ideal $\langle\sigma_{\vB}(w)\rangle$ is radical.
\end{lemma}
\begin{proof}
Note that $\mathcal{I}(\LX_{\infty})=\langle w\rangle$ is radical.
So, by \cite[page 180, Proposititon 9]{CLO2015},
$w$ is square-free in $\Q[\vU]$.
Then, by Theorem \ref{thm:Hilbert},
there exists a Zariski dense subset $\Theta\subseteq\Q^{n-1}$ and an affine variety $V\subset\C^{n-1}$ such that for any $\vB\in\Theta\setminus V$,
$\sigma_{\vB}(w)$ is square-free in $\Q[u_1]$.
Then, by \cite[page 180, Proposititon 9]{CLO2015}, we complete the proof.
\end{proof}



\noindent\textbf{Proof of Corollary \ref{coro:Sample}.}
\begin{proof}
Let $\Delta$ be the affine variety generated by
$\lc_{\prec{u_1,\vX}}\langle\vf\rangle$,
where 
$\prec_{u_1,\vX}$ denotes the block monomial order such that $u_1\ll\vX$ and for $\Vector{x}$, we apply  the graded reverse lex order with $x_1>\cdots>x_m$,  and
${\rm lc}_{\prec_{u_1,\vX}}\langle\vf\rangle$ denotes the set of the leading coefficients of all polynomials in $\langle\vf\rangle$ w.r.t. $\prec_{u_1,\vX}$.
For any $\vB\in\Q^{n-1}\setminus\Delta$,
by Theorem \ref{thm:main} and Lemma \ref{lemma:map-sample-onto},
we have
\begin{align*}
    \mI(\mV(\sigma_{\vB}(\vf))_{\infty})\;=\;\sqrt{\sigma_{\vB}(\langle w\rangle)}.
\end{align*}
By Lemma \ref{lemma:homo-ideal} and Lemma \ref{lemma:map-sample-onto}, $\sigma_{\vB}(\langle w\rangle)=\langle\sigma_{\vB}(w)\rangle$.
By Lemma \ref{lemma:sample-radical},
there exists a Zariski dense subset $\Theta\subseteq\Q^{n-1}$ and an affine variety $V_0\subset\C^{n-1}$ such that for any $\vB\in\Theta\setminus V_0$,
the ideal $\langle\sigma_{\vB}(w)\rangle$ is radical.
Let $V:=\Delta\cup V_0$, and we complete the proof.
\end{proof}

\section{Algorithm}\label{sec:algorithms}
 In Section \ref{sec:LikelihoodEquations},
we recall the basic concepts related to the likelihood equations.
In Section \ref{sec:standard}, we recall a standard method (Algorithm \ref{alg:standard}) for computing nonproperness sets.
In Section \ref{sec:computing},  we  present a new algorithm (Algorithm \ref{alg:new}) for computing nonproperness sets of likelihood equations.
We prove the correctness of the new algorithm in Section \ref{subsec:correct-new}. Finally,  we illustrate how the new algorithm works by an example in Section \ref{subsec:examples}.

\subsection{Likelihood Equations}\label{sec:LikelihoodEquations}

\begin{definition}[{\bf Algebraic Statistical Model}]\label{def:StatisticalModel}
Given homogeneous polynomials $g_1,\ldots,g_s$ in ${\mathbb Q}[p_1, \ldots,p_n]$ such that
\[V\;:=\;\{(a_1,\ldots,a_n)\in {\mathbb C}^{n}\mid g_i(a_1,\ldots,a_n)=0\;\text{for all}\;1\le i\le s\}\]
is irreducible and generically reduced, we define an \struc{\textit{algebraic statistical model}} as
$\cM:=V\cap \Delta_{n}$, where
\[\Delta_{n}:=\{(p_1, \ldots, p_n)\in {\mathbb R}^{n} \ |\ p_1>0, \ldots,p_n>0,\;p_1+\cdots+p_n=1\}.\]
If $g_1, \ldots, g_s$ are algebraically  independent, then we say $\{g_1, \ldots, g_s\}$ is a set of \struc{\textit{independent model invariants}}.
\end{definition}

Given an algebraic statistical model $\cM$ with independent model invariants $g_1, \ldots, g_s$ and a data vector $\Vector{u} := (u_1, \ldots, u_n)\in {\mathbb R}_{\geq 0}^{n}$,
the \struc{\textit{maximum likelihood estimation} (MLE)} problem in statistics is to maximize a likelihood function as following:
\begin{align*}
    \begin{aligned}
    	\max \ \Pi_{k=1}^np_k^{u_k} \ \text{ subject to } (p_1,\ldots,p_n) \in \cM.
    \end{aligned}
\end{align*}
 Recall the polynomial set $\Vector{f}=\{f_1, \ldots, f_{n+s+1}\}$ defined in \eqref{eq:lle}. By the Lagrange multiplier method,
for any critical point $\vp^*:=(p^*_1, \ldots, p^*_n)\in{\mathbb R}_{> 0}^{n}$ to  the above MLE problem, there exists $\Vector{\lambda}^*:=(\lambda_1^*,\ldots,\lambda_{s+1}^*)\in {\mathbb R}^{s+1}$ such that $f_1(\Vector{u},\vp^*,\Vector{\lambda}^*)=\cdots=f_{n+s+1}(\Vector{u},\vp^*,\Vector{\lambda}^*)=0$. So, we can solve all critical points to MLE problem by solving the equations
$f_1=\cdots=f_{n+s+1}=0$ when $\Vector{u}$ is given.
Formally, we have the following definition.
\begin{definition}[\bf Lagrange Likelihood Equations]\label{def:LikelihoodEquations}
Given an algebraic statistical model $\cM$ with independent model invariants $g_1, \ldots, g_s \in  {\mathbb Q}[p_1, \ldots,p_n]$, the polynomial set $\Vector{f}=\{f_1, \ldots, f_{n+s+1}\}$ defined in \eqref{eq:lle}
is said to be the system of \struc{\textit{Lagrange likelihood equations}} (or, simply \struc{\textit{likelihood equations}} in this paper) of $\cM$ when set equal to zeros.
\end{definition}

According to \citep{SAB2005}, for a generic choice of data $\Vector{u}$, the system of Lagrange likelihood equations $\vf$ defined as in \eqref{eq:lle} admits finitely many complex solutions; that is, there exists a proper affine variety $V\subsetneq {\mathbb C}^{n}$ such that for any $\Vector{u}^*\in {\mathbb C}^{n}\setminus V$, the equations $f_1(\vU^*, \Vector{p}, \Vector{\lambda})=\cdots=f_{n+s+1}(\vU^*, \Vector{p}, \Vector{\lambda})=0$ have exactly $N$ common complex solutions in ${\mathbb C}^{n+s+1}$. As mentioned in Example \ref{ex:symmetric}, this non-negative integer $N$ is called the ML-degree of the algebraic statistical model $\cM$.

\subsection{Standard Algorithm}\label{sec:standard}
For any likelihood-equation system  $\Vector{f}$ $\subseteq {\mathbb Q}[\Vector{u}, \Vector{p}, \Vector{\lambda}]$ (recall \eqref{eq:lle}),
we make the following assumption, which is milder enough in applications.
\begin{assumption}\label{ap:principal}
  \begin{align}\label{eq:ap-principal}
      \mI(\mV(\vf)_{\infty})\;=\;\langle w\rangle,\;{\rm where}\;w\in{\mathbb Q}[\Vector{u}].
  \end{align}
\end{assumption}
\noindent

The goal of this work is to compute the polynomial $w$ in \eqref{eq:ap-principal} (i.e., the generator of the nonproperness set) for a given system of likelihood equations $\vf$ \eqref{eq:lle}. In order to make the notation simpler, we rename the variables $p_1,\ldots,p_n,\lambda_1,\ldots,\lambda_s$  in  $\vf$ as $x_1,\ldots,x_m$. Then,
$\vf$ can be considered as a subset of ${\mathbb Q}[\vU, \vX]$.
We revisit  a standard method
for computing nonproperness sets of general parametric  systems in ${\mathbb Q}[\vU, \vX]$,  see Algorithm \ref{alg:standard}. Of course, this method can be applied to likelihood equations. However,  the standard method only defeats small models efficiently since its first step is to compute the reduced Gr\"obner bases, see  the computation timings for the benchmarks shown in Table \ref{table:literatureOld}.

\begin{algorithm}[ht]
	\scriptsize
	\DontPrintSemicolon
	\LinesNumbered
	\SetKwInOut{Input}{Input}
	\SetKwInOut{Output}{Output}
	\Input{ likelihood-equation system $\vf\subseteq{\mathbb Q}[\Vector{u},\Vector{x}]$, parameters  $\Vector{u}$,  variables $\Vector{x}$}
	\Output{ $w\in\Q[\vU]$, a generator polynomial of ${\mathcal{I}}(\LX_{\infty})$}
	\caption{ {\bf StandardMethod}\\ \cite[Algorithm {PROPERNESSDEFECTS}]{DV2005}}\label{alg:standard}
	\BlankLine
	 Compute the reduced Gr\"obner basis of the ideal $\langle \vf \rangle$ w.r.t. $\prec_{\vU,\vX}$, say $G$, where $\prec_{\vU,\vX}$ is an admissible block monomial order such that $\vU\ll\vX$
	\nllabel{algline:GRLGrobner}\;
	\For{$i$ {\bf from} $1$ {\bf to} $m$ }
	{$\#${\tt Recall that $m$ denotes the number of variables $x_1,\ldots, x_m$}\;
    $C_i\leftarrow G\cap {\mathbb Q}[\Vector{u}]$\; 
		\For{{\rm every} $g$ {\rm in} $G$}
		{\If {{\rm the leading monomial of} $g$ w.r.t $\prec_{\vX}$ is $x_i^k$ {\rm for some} $k\;(k\in\Z_{\ge1})$}
			{$C_i\leftarrow C_i\cup\{{\rm lc}_{\pX}(g)\}$}
		}
	}
	$w\leftarrow$ the  generator polynomial of the ideal $\sqrt{\cap^{m}_{i=1}\langle C_i\rangle}$\;
	{\bf return} $w$\;
\end{algorithm}

\subsection{New Algorithm}\label{sec:computing}

The main result of this work is a probabilistic algorithm  for computing the generator $w$ of $\mI(\LX_{\infty})$ in \eqref{eq:ap-principal},
which is a specialization/interpolation method inspired by \citep{Tang2017}, see Algorithm \ref{alg:new}. First, we give a brief description of our method.
Without loss of generality,
for every irreducible factor of $w$, we assume that all its  coefficients are in ${\mathbb Z}$,   and we assume that the greatest common divisor (GCD) of its coefficients is $1$.
In fact, we can write $w$ as a product of three factors:
\begin{align}\label{eq:factors}
w\;=\;LFactor^{(1)}\cdot NLFactor\cdot LFactor^{(2)},
\end{align}
where
\begin{itemize}
\item $LFactor^{(1)}\in {\mathbb Z}[\vU]$ denotes the product of all linear  factors in which all nonzero coefficients are 1,
\item  $NLFactor\in {\mathbb Z}[\vU]$ denotes the product of all nonlinear irreducible factors, and
\item $LFactor^{(2)}\in {\mathbb Z}[\vU]$ denotes the product of all linear  factors in which not all nonzero coefficients are 1.
\end{itemize}

\begin{example}
For instance, if $w=(u_1+u_2)(u_2+u_3)(u_1u_2+u_2u_3)(2u_1+u_2+3u_3)$, then
$LFactor^{(1)}=(u_1+u_2)(u_2+u_3)$, $NLFactor=u_1u_2+u_2u_3$, and $LFactor^{(2)}=2u_1+u_2+3u_3$.
\end{example}

\begin{remark}\label{rmk:homo}
For the generator polynomial $w$ shown in Assumption \ref{ap:principal}, we have the following remarks.
\begin{itemize}
\item Notice that $w$ in \eqref{eq:ap-principal} is square-free since it generates a radical ideal.
\item  By \cite[page 347]{Tang2017}, the polynomial $w$ in  is homogeneous.
So, every  factor of $w$ is homogeneous.
\end{itemize}
\end{remark}

The basic idea is to interpolate the three factors in \eqref{eq:factors} one by one but we need to first apply a coordinate change to $u_1, \ldots, u_n$.
We present an outline of the new algorithm as the follows.

\begin{enumerate}
\item[{\bf Step 1}]
We introduce new parameters $v_1, \ldots, v_n$, and we apply the invertible linear transformation $\ta$ (recall \eqref{eq:def-linear-map}) to $\vf$:
\begin{equation}\label{eq:rewrite-linear}
    u_1 = v_1, \;\;\text{and} \;\; u_j = v_j + a_{j-1}v_1\;\;\text{for}\;\; j=2,\ldots,n,
\end{equation}
where $\vA=(a_1,\ldots,a_{n-1})$ is  rational vector.
By Corollary \ref{coro:Linear},
\begin{align}
    \mI(\mV(\ta(\vf))_{\infty})\;=\;\langle\ta(w)\rangle.
\end{align}
Notice that once we know $\ta(w)$, it is straightforward to recover $w$ since $\ta$ is invertible. So, we only need to compute $\ta(w)$.
Applying the homomorphism $\ta$ on both sides of
the equality \eqref{eq:factors}, we get
\begin{align}\label{eq:lcfactors}
\ta(w)\;=\;\ta(LFactor^{(1)})\cdot\ta(NLFactor)\cdot\ta(LFactor^{(2)}).
\end{align}
In the following steps, we compute  $\ta(LFactor^{(1)})$, $\ta(NLFactor)$, and $\ta(LFactor^{(2)})$ respectively.

\item[{\bf Step 2}] We compute $\ta(LFactor^{(1)})$ by ``sampling" only once. Here, by ``sampling" we mean  specializing some parameters as concrete rational numbers and computing a Gr\"obner basis.
We complete this step in Algorithm \ref{alg:allone},
the correctness of which will be proved in Lemma \ref{lemma:alg-allone-correct} (see Section \ref{subsubsec:correct-allone}).

\item[{\bf Step 3}]
In order to interpolate $\ta(NLFactor)$ and $\ta(LFactor^{(2)})$,
we compute
\begin{itemize}
    \item $\deg(\ta(w),v_i)$ for every $i\in\{1,\ldots,n\}$,
    \item $\deg(\ta(NLFactor),v_1)$, and
    \item an upper bound for  $\deg(\ta(NLFactor),v_i)$ for every $i\in\{2,\ldots,n\}$.
\end{itemize}
We complete the step in Algorithm \ref{alg:degree},
the correctness of which will be proved in Lemma \ref{lemma:alg-degree-correct} (see Section \ref{subsubsec:correct-degree}).

\item[{\bf Step 4}] We interpolate $\ta(NLFactor)$ and $\ta(LFactor^{(2)})$ by a similar way to the specialization/linear-lifting method used in   {\cite[Strategy 1]{Tang2017}}.
Here, we modify the method as Algorithm \ref{alg:nonlinear}. The correctness will be proved in Lemma \ref{lemma:alg-nonlinear-correct} (see Section \ref{subsubsec:correct-nonlinear}).
\end{enumerate}

We present the pseudocode of the new algorithm below, see Algorithm \ref{alg:new} with three sub-algorithms.
The termination of Algorithm \ref{alg:new} is clear since Algorithm \ref{alg:new} only has finite loops.
From the above outline,  every step corresponds to a sub-algorithm excluding  Step 1. If all sub-algorithms are correct, then Algorithm \ref{alg:new} is clearly correct.  For each   sub-algorithm,  we explain the details  and  prove the correctness in Section \ref{subsec:correct-new}.

\begin{algorithm}[ht]
\scriptsize
\DontPrintSemicolon
\LinesNumbered
\SetKwInOut{Input}{Input}
\SetKwInOut{Output}{Output}
\Input{ likelihood-equation system $\vf\subseteq{\mathbb Q}[\Vector{u},\Vector{x}]$, parameters $\Vector{u}$, variables  $\Vector{x}$}
\Output{ $w\in\Q[\vU]$, a generator polynomial of ${\mathcal{I}}(\LX_{\infty})$}
\caption{ {\bf NewMethod} ({\bf Main Algorithm})}\label{alg:new}
\BlankLine
Choose a  rational vector $\vA:=(a_1,\ldots,a_{n-1})$.  Obtain $\ta(\vf)$ by replacing $u_1,u_2,\ldots, u_n$ in $\vf$ with $v_1,v_2+a_1v_1,\ldots,v_{n}+a_{n-1}v_1$ \;
Compute $\ta(LFactor^{(1)})$ in \eqref{eq:lcfactors} by calling {\bf AllOne}$(\ta(\vf), \vV, \vX, \vA)$\;
$wDegree,\;NLFactorDegreeBound\leftarrow$ {\bf Degrees}$(\ta(\vf), \vV, \vX)$\;
Compute $\ta(NLFactor)$ and $\ta(LFactor^{(2)})$ in \eqref{eq:lcfactors} by calling {\bf Interpolation}$(\ta(\vf), \vV, \vX, \ta(LFactor^{(1)}), wDegree, NLFactorDegreeBound)$\;
$w\leftarrow$ apply the inverse linear transformation of $\ta$ to $\ta(LFactor^{(1)})\cdot \ta(NLFactor)\cdot \ta(LFactor^{(2)})$\;
{\bf return}  $w$\;
\end{algorithm}

\begin{algorithm}[ht]
\scriptsize
\DontPrintSemicolon
\LinesNumbered
\SetKwInOut{Input}{Input}
\SetKwInOut{Output}{Output}
\Input{ likelihood-equation system after the linear transformation $\ta(\vf)\subseteq{\mathbb Q}[\Vector{v},\Vector{x}]$, parameters $\Vector{v}$, variables $\Vector{x}$, the rational vector $\vA=(a_1,\ldots,a_{n-1})$ corresponding to $\ta$}
\Output{ $\ta(LFactor^{(1)})$ in \eqref{eq:lcfactors}}
\caption{ {\bf AllOne (Sub-Algorithm of Algorithm \ref{alg:new})}}\label{alg:allone}
\BlankLine
Choose a rational vector $\Vector{b}:=(b_1,\ldots,b_{n-1})$. Obtain $\sigma_{\Vector{b}}\cdot\ta(\vf)$ by specializing $v_2,\ldots, v_n$ in $\ta(\vf)$ as $b_1,\ldots, b_{n-1}$  \nllabel{line:allone-1}\;
$w^*\leftarrow$ {\bf StandardMethod}$(\sigma_{\Vector{b}}\cdot\ta(\vf), v_1, \Vector{x})$\nllabel{line:allone-3}
$\#${\tt We  always make $w^*$ a monic polynomial}\;
\For {{\rm every nonempty subset} $T\subseteq\{1,\ldots, n\}$}{ {\bf if} $1\in T$, {\bf then} $\chi_{T}\leftarrow 1+\sum_{j\in T\setminus\{1\}}a_{j-1}$; {\bf else}, $\chi_{T}\leftarrow \sum_{j\in T}a_{j-1}$.\nllabel{line:allone-9}\;}
$\ell\leftarrow1$\nllabel{line:allone-4}\;
\For{{\rm every linear factor} $v_1+c$ $(c\in\Q)$ {\rm of} $w^*$\nllabel{line:allone-5}}{
    \If{{\rm there exists a nonempty subset} $T\subseteq\{1,\ldots, n\}$ {\rm such that} $c=\frac{1}{\chi_{T}}\sum_{j\in T\setminus\{1\}}b_{j-1}$\nllabel{line:allone-7}}
    {$\ell\leftarrow\ell\cdot\left(\chi_{T}v_1+\sum_{j\in T\setminus\{1\}}v_j\right)$\nllabel{line:allone-8}\;}
}
{\bf return} $\ell$
\end{algorithm}

\begin{algorithm}[ht]
\scriptsize
\DontPrintSemicolon
\LinesNumbered
\SetKwInOut{Input}{Input}
\SetKwInOut{Output}{Output}
\Input{ likelihood-equation system after the linear transformation $\ta(\vf)\subseteq{\mathbb Q}[\Vector{v},\Vector{x}]$, parameters $\Vector{v}$, variables $\Vector{x}$}
\Output{ $wDegree$ and $NLFactorDegreeBound$ where
	\begin{itemize}
	    \item $wDegree$ is a list, whose $j$-th entry is $\deg(\ta(w), v_j)$ for $j=1,\ldots,n$,
	    \item $NLFactorDegreeBound$ is a list, whose
	    first entry is $\deg(\ta(NLFactor),v_1)$ and \\
	    $j$-th entry is an upper bound for $\deg(\ta(NLFactor),v_j)$ for $j=2,\ldots,n$
	\end{itemize}
	}
\caption{ {\bf Degrees (Sub-Algorithm of Algorithm \ref{alg:new}) }}\label{alg:degree}
\BlankLine
\For{$j$ {\bf from} $1$ {\bf to} $n$\nllabel{line:degrees-1}}{
$\Vector{b}\leftarrow(b_{1},\ldots,b_{n-1})$, where every $b_{i}$ is a  rational number\nllabel{line:degrees-2}\;
$\vf^*\leftarrow\ta(\vf)|_{v_1=b_1,\ldots,v_{j-1}=b_{j-1},v_{j+1}=b_{j},\ldots,v_{n}=b_{n-1}}$\;
$w_j^*\leftarrow$ {\bf StandardMethod}$(\vf^*, v_j, \Vector{x})$ $\#${\tt We  always make $w_j^*$ a monic polynomial} \nllabel{line:degrees-3}\;
$wDegree[j]\leftarrow\deg(w_j^*, v_j)$\nllabel{line:degrees-4}\;
$\eta_j\leftarrow$ the product of nonlinear irreducible factors of $w_j^*$ \nllabel{line:degrees-5}\;
$k_j\leftarrow$ the number of nonlinear irreducible factors of $w_j^*$\nllabel{line:degrees-7}\;
$NLFactorDegreeBound[j]\leftarrow\deg(\eta_j, v_j)+(k_1-k_j)$\nllabel{line:degrees-9}\;
    }
{\bf return} $wDegree,\;NLFactorDegreeBound$\;
\end{algorithm}

\begin{algorithm}[ht]
\scriptsize
\DontPrintSemicolon
\LinesNumbered
\SetKwInOut{Input}{Input}
\SetKwInOut{Output}{Output}
\Input{ likelihood-equation system after the linear transformation $\ta(\vf)\subseteq{\mathbb Q}[\Vector{v},\Vector{x}]$, parameters $\Vector{v}$, variables $\Vector{x}$, $\ta(LFactor^{(1)})$ in \eqref{eq:lcfactors}, and
$wDegree$, $NLFactorDegreeBound$ where
	\begin{itemize}
	    \item $wDegree$ is a list, whose $j$-th entry is $\deg(\ta(w), v_j)$ for $j=1,\ldots,n$,
	    \item $NLFactorDegreeBound$ is a list, whose
	    first entry is $\deg(\ta(NLFactor),v_1)$ and \\
	    $j$-th entry is an upper bound for $\deg(\ta(NLFactor),v_j)$ for $j=2,\ldots,n$
	\end{itemize}}
\Output{ $\ta(NLFactor)$ and $\ta(LFactor^{(2)})$ in \eqref{eq:lcfactors}}
\caption{ {\bf Interpolation (Sub-Algorithm of Algorithm \ref{alg:new}) }}\label{alg:nonlinear}
\BlankLine
$d_1\leftarrow NLFactorDegreeBound[1]$\nllabel{line:nonlinear-1}\;
$d_2\leftarrow wDegree[1]-\deg(\ta(LFactor^{(1)}),v_1)-d_1$\;
\For{$i$ {\bf from} $1$ {\bf to} $d_1$\nllabel{line:nonlinear-2}}
{
$U_{i,1},\ldots,U_{i,N_i}\leftarrow$ all monomials $U\in\Q[v_2,\ldots,v_n]$ with $\deg(U)=i$ and $\deg(U,v_j)\le NLFactorDegreeBound[j]$\nllabel{line:nonlinear-3}
}
$N\leftarrow\max(N_1,\ldots,N_{d_1})$\nllabel{line:nonlinear-4}\;
\For{$k$ {\bf from} $1$ {\bf to} $N$}{
$\Vector{b}\leftarrow(b_{k,1},\ldots,b_{k,n-1})$, where every
$b_{k,r}$ is a  rational number \nllabel{line:nonlinear-6}\;
$w^*\leftarrow$ {\bf StandardMethod}$(\sigma_{\Vector{b}}\cdot\ta(\vf), v_1, \Vector{x})$ $\#${\tt We  always make $w^*$ a monic polynomial}\nllabel{line:nonlinear-7}\;
$\eta^*\leftarrow$ the product of nonlinear irreducible factors of $w^*$\nllabel{line:nonlinear-8}\;
$C_{1,k}^*,\ldots,C_{d_1,k}^*\leftarrow\coeff(\eta^*,v_1^{d_1-1}),\ldots,\coeff(\eta^*,v_1^0)$\;
$\zeta^*\leftarrow$ the product of linear factors of $w^*$ but not of $\sigma_{\vB}\cdot\ta(LFactor^{(1)})$\nllabel{line:nonlinear-10}\;
$D_{1,k}^*,\ldots,D_{d_2,k}^*\leftarrow\coeff(\zeta^*,v_1^{d_2-1}),\ldots,\coeff(\zeta^*,v_1^0)$\;
}
\For{$i$ {\bf from} $1$ {\bf to} $d_1$}{
${\mathcal M}_i\leftarrow N_i\times N_i$ matrix whose $(k,r)$-entry is $U_{i,r}(b_{k,1},\ldots,b_{k,n-1})$\;
$C_i\leftarrow(U_{i,1},\ldots,U_{i,N_i}){\mathcal M}_i^{-1}(C_{i,1}^*,\ldots,C_{i,N_i}^*)^{\top}$\nllabel{line:nonlinear-12}
}
$q_1\leftarrow v_1^{d_1}+\sum_{i=1}^{d_1}C_{i}\cdot v_1^{d_1-i}$\nllabel{line:nonlinear-16}\;
\For{$i$ {\bf from} $1$ {\bf to} $d_2$}
{
$V_{i,1},\ldots,V_{i,M_i}\leftarrow$ all monomials $U\in\Q[v_2,\ldots,v_n]$ with $\deg(U)=i$ and $\deg(U,v_j)\le wDegree[j]-\deg(\ta(LFactor^{(1)}),v_j)-\deg(q_1,v_j)$
}
$M\leftarrow\max(M_1,\ldots,M_{d_2})$\nllabel{line:nonlinear-21}\;
\If{$M> N$\nllabel{line:max_time}}{
\For{$k$ {\bf from} $N+1$ {\bf to} $M$}{
$\Vector{b}\leftarrow(b_{k,1},\ldots,b_{k,n-1})$, where every
$b_{k,r}$ is a  rational number \nllabel{line:nonlinear-22}\;
$w^*\leftarrow$ {\bf StandardMethod}$(\sigma_{\Vector{b}}\cdot\ta(\vf), v_1, \Vector{x})$ $\#${\tt We  always make $w^*$ a monic polynomial}\;
$\zeta^*\leftarrow$ the product of linear factors of $w^*$ but not of $\sigma_{\vB}\cdot\ta(LFactor^{(1)})$\nllabel{line:nonlinear-24}\;
$D_{1,k}^*,\ldots,D_{d_2,k}^*\leftarrow\coeff(\zeta^*,v_1^{d_2-1}),\ldots,\coeff(\zeta^*,v_1^0)$\;
}
}
\For{$i$ {\bf from} $1$ {\bf to} $d_2$}{
${\mathcal M}_i\leftarrow M_i\times M_i$ matrix whose $(k,r)$-entry is $V_{i,r}(b_{k,1},\ldots,b_{k,n-1})$\;
$D_i\leftarrow(V_{i,1},\ldots,V_{i,M_i}){\mathcal M}_i^{-1}(D_{i,1}^*,\ldots,D_{i,M_i}^*)^{\top}$  \nllabel{line:nonlinear-28}
}
$q_2\leftarrow v_1^{d_2}+\sum_{i=1}^{d_2}D_{i}\cdot v_1^{d_2-i}$\nllabel{line:nonlinear-29}\;
{\bf return} $q_1,q_2$\;
\end{algorithm}

\subsection{Correctness of New Algorithm} \label{subsec:correct-new}

We prove the correctness of three sub-algorithms Algorithms \ref{alg:allone}--\ref{alg:nonlinear} in this section. Since there are many  technical details, one can also skip the proofs and go to Section \ref{subsec:examples} for running examples.
Notices that Algorithms \ref{alg:allone}--\ref{alg:nonlinear} are all probabilistic algorithms. When we say a probabilistic algorithm is correct, we mean the algorithm is correct with probability $1$. For instance, Algorithm
\ref{alg:allone} is correct if the input  vector $\vA \in{\mathbb Q}^{n-1}$ and the vector $\Vector{b}\in{\mathbb Q}^{n-1}$ chosen in Line \ref{line:allone-1} are indeed generic, which means they do not belong to certain algebraic varieties with dimensions strictly lower than $n-1$.   One will see what these algebraic varieties are for each sub-algorithm in the proofs.

\subsubsection{Correctness of Algorithm \ref{alg:allone}}\label{subsubsec:correct-allone}


In order to prove the correctness of Algorithm \ref{alg:allone}, we first prepare some lemmas.
Since in the rest of the paper, we need to consider
the composition of the two homomorphisms $\sigma_{\vB}$
and $\tau_{\vA}$, we clarify that for any $\vB:=(b_1,\ldots,b_{n-1})\in\Q^{n-1}$,  we still denote by $\sigma_{\vB}$ the homomorphism
$ \sigma_{\vB}:\Q[\vV,\vX]\rightarrow\Q[v_1,\vX]$ such that
\begin{align}\label{eq:def-sample-mapv}
    \sigma_{\vB}(g(\vV,\vX))\;=\;g(v_1,b_1, \ldots,b_{n-1},\vX)\;\text{for any}\; g\in \Q[\vV,\vX].
\end{align}
Note that the homomorphism \eqref{eq:def-sample-mapv} is exactly $\sigma_{\vB}$ defined in \eqref{eq:def-sample-map} (the only difference is that here we consider the homomorphism on the polynomial ring $\Q[\vV,\vX]$ instead of $\Q[\vU,\vX]$).

\begin{lemma}\label{lemma:alg-allone-pre}
Let $\vf\subseteq {\mathbb Q}[\vU, \vX]$. If $\mI(\mV(\vf)_{\infty})=\langle w\rangle$, where $w\in{\mathbb Q}[\Vector{u}]$, then
for any $\vA\in\Q^{n-1}$,
there exists a Zariski dense subset $\Theta_{\vA}\subseteq\Q^{n-1}$ and an affine variety in $\C^{n-1}$, say $B_{\vA}$,
such that
for any $\vB\in\Theta_{\vA}\setminus B_{\vA}$,
\begin{align}\label{eq:apply-two-coros-1}
\mI(\mV(\sigma_{\Vector{b}}\cdot\ta(\vf))_{\infty})\;=\;\langle \sigma_{\Vector{b}}\cdot\ta(w)\rangle.
\end{align}

\end{lemma}
\begin{proof}
For any $\vA\in\Q^{n-1}$, by Corollary \ref{coro:Linear}, we have
$\mI(\mV(\ta(\vf))_{\infty})=\langle\ta(w)\rangle$.
Applying Corollary \ref{coro:Sample} to $\ta(\vf)$, the proof is done.
\end{proof}

\begin{lemma}\label{lemma:degree-correct-pre}
For any $q\in\Q[\vU]$, there exists an affine variety in $A\subset\C^{n-1}$
such that for any $\vA\in {\mathbb Q}^{n-1}\setminus A$ and for every factor $\xi$ of $q$,
$\ta(\xi)$ is a factor of $\ta(q)$ satisfying
\begin{align*}
    \deg(\ta(\xi),v_1)=\deg(\ta(\xi),\vV)=\deg(\xi,\vU).
\end{align*}
Furthermore, if $\xi$ is  irreducible, then
$\ta(\xi)$ is also irreducible.
\end{lemma}
\begin{proof}

 For any $q\in\Q[\vU]$,
let $d:=\deg(q,\vU)$.
Suppose  
\begin{align}\label{eq:w-polynomial}
    q\;=\;\sum_{|\alpha|=d}C_{\alpha} u_{1}^{\alpha_1}\cdots u_{n}^{\alpha_n}+q_0,
\end{align}
where $\alpha=(\alpha_1,\ldots,\alpha_n)\in\Z^{n}_{\ge0}$, $|\alpha|:=\Sigma^n_{i=1}\alpha_i$, $C_{\alpha}\in\Q$ and $q_0\in\Q[\vU]$ with $\deg(q_0,\vU)\le d-1$.
Let
\begin{align}\label{eq:linear-set}
    A:=\;\{(a_1,\ldots,a_{n-1})\in\C^{n-1}\mid\sum_{|\alpha|=d}C_{\alpha}a_1^{\alpha_2}\cdots a_{n-1}^{\alpha_n}=0\}.
\end{align}
Notice that $A\neq \C^{n-1}$ once $q$ is not the zero polynomial.
In the argument below, we assume that
$\vA\in {\mathbb Q}^{n-1}\setminus A$.
First, we prove that $\deg(\ta(q),v_1)=\deg(\ta(q),\vV)=\deg(q,\vU)$.
Note that $\deg(\ta(q),v_1)\le\deg(\ta(q),\vV)$.
Applying Lemma \ref{lemma:linear-irreducible} $(\romannumeral1)$ to $q$, we have $\deg(\ta(q),\vV)\le \deg(q,\vU)$.
Thus, $\deg(\ta(q),v_1)\le\deg(\ta(q),\vV)\le\deg(q,\vU)$.
So, we only need to prove that $\deg(\ta(q),v_1)=\deg(q,\vU)$.
By \eqref{eq:w-polynomial} and by the definition of $\ta$ in \eqref{eq:def-linear-map}, we have
\begin{align}    \ta(q)\;&=\;\sum_{|\alpha|=d}C_{\alpha}\cdot v_1^{\alpha_1}(v_2+a_1v_1)^{\alpha_2}\cdots (v_n+a_{n-1}v_1)^{\alpha_n}+\ta(q_0)\nonumber\\    &=\;\sum_{|\alpha|=d}\left(C_{\alpha}a_1^{\alpha_2}\cdots a_{n-1}^{\alpha_n}v_1^{d}+
    C_{\alpha}\sum_{i=\alpha_1}^{d-1} P^{(i)}_{\alpha}v_1^{i}\right)+\ta(q_0),\label{eq:w-up-to-scaling-pre}
\end{align}
where $d=\deg(q,\vU)$ and $P_{\alpha}^{(i)}\in\Q[v_2,\ldots,v_n]$.
Applying Lemma \ref{lemma:linear-irreducible} $(\romannumeral1)$ to $q_0$, we have $\deg(\ta(q_0),\vV)\le \deg(q_0,\vU)$.
Recall that $\deg(q_0,\vU)\le d-1$.
So, $\deg(\ta(q_0),\vV)\le d-1$, and thus $\deg(\ta(q_0),v_1)\le d-1$.
Then, by the definition of $A$ and by \eqref{eq:w-up-to-scaling-pre}, we have $\deg(\ta(q),v_1)=d$.

Below, we prove the conclusion.
For every factor $\xi$ of $q$, $\ta(\xi)$ is a factor of $\ta(q)$ since $\ta$ is a homomorphism.
Similar to the proof in the above paragraph, we have $\deg(\ta(\xi),v_1)\le\deg(\ta(\xi),\vV)\le\deg(\xi,\vU)$.
Then, since $\deg(\ta(q),v_1)=\deg(\ta(q),\vV)\allowbreak=\deg(q, \vU)$,
we have
\begin{align}\label{eq:deg-xi-equal}
    \deg(\ta(\xi),v_1)=\deg(\ta(\xi),\vV)=\deg(\xi, \vU).
\end{align}
Furthermore, if $\xi$ is an irreducible factor of $q$,
by Lemma \ref{lemma:linear-irreducible} ($\romannumeral2$), $\ta(\xi)$ is an irreducible factor of $\ta(q)$.
We complete the proof.
\end{proof}

\begin{lemma}[Algorithm \ref{alg:allone}: AllOne]
\label{lemma:alg-allone-correct}
The probabilistic algorithm Algorithm \ref{alg:allone} is correct.
\end{lemma}
\begin{proof}
Suppose $\vf$ is the input system of likelihood equations.  Recall that we have assumed $\mI(\mV(\vf)_{\infty})\;=\;\langle w\rangle$
(i.e., Assumption \ref{ap:principal} holds).
By Lemma \ref{lemma:alg-allone-pre}, Lemma  \ref{lemma:degree-correct-pre}, and Theorem \ref{thm:Hilbert},
there exists an affine variety $A\subset {\mathbb C}^{n-1}$ such that for any $\vA\in\Q^{n-1}\setminus A$, there exists a Zariski dense subset $\Theta_{\vA}$ and an affine variety $B_{\vA}\subset\C^{n-1}$
such that for any $\Vector{b}\in \Theta_{\vA}\setminus B_{\vA}$, we have \eqref{eq:apply-two-coros-1} holds, and for
every linear factor $\xi$ of $w$, $\tau_{\vA}(\xi)$ is  a linear factor of $\tau_{\vA}(w)$, and $\sigma_{\vB}\cdot\tau_{\vA}(\xi)$ is  a linear factor of
$\sigma_{\vB}\cdot\tau_{\vA}(w)$.

Let
\begin{align}
    S_1\;&:=\;\{\eta\mid\eta~{\rm is~a~linear~factor~of}~w\},~{\rm and}\\
    S_2\;&:=\;\{\eta\mid\eta=\sum_{i\in T}u_i,~{\rm where}~T~{\rm is~a~nonempty~subset~of~}\{1,\ldots,n\}\}.
\end{align}
Let $S:=S_1\cup S_2$.
Define
\begin{align}
    V_{\vA}\;:=\;\{\vA\in\Q^{n-1}\mid~{\rm there~ exists}~\eta\in S_2~{\rm such ~ that}~\coeff(\ta(\eta),v_1)=0\},
\end{align}
where $\coeff(\ta(\eta),v_1)$ denotes the coefficient of $\ta(\eta)$ w.r.t.\;$v_1$.
Notice that if $\eta\in S_1$, from the proof of Lemma \ref{lemma:degree-correct-pre}, we know  that for any $\vA\in\Q^{n-1}\setminus A$, $\coeff(\ta(\eta),v_1)\ne0$. Therefore, for any $\vA\in\Q^{n-1}\setminus (A\cup V_{\vA})$, $\coeff(\ta(\eta),v_1)\ne0$ for any $\eta\in S$.
In the argument below,
we assume that $\vA\in\Q^{n-1}\setminus(A\cup V_{\vA})$.
Define
\begin{align}
    V_{\vB}:=\;\{\vB\in\Q^{n-1}\mid~&{\rm there\;exist}\;\eta_1,\eta_2\in S\;(\eta_1\neq \eta_2)\;{\rm such\;that}\\
    &\sigma_{\vB}
    (\frac{\ta(\eta_1)}{\coeff(\ta(\eta_1),v_1)}-\frac{\ta(\eta_2)}{\coeff(\ta(\eta_2),v_1)})=0\}.
\end{align}
In the argument below,
we further assume that  $\vB\in \Theta_{\vA}\setminus (B_{\vA}\cup V_{\vB})$. 
Using Algorithm \ref{alg:standard}, we can compute the monic polynomial $w^*\in\Q[v_1]$ (see Algorithm \ref{alg:allone}--Line \ref{line:allone-3}) such that
\begin{align}\label{eq:apply-alg:intersect}
    \mI(\mV(\sigma_{\Vector{b}}\cdot\ta(\vf))_{\infty})\;=\;\langle w^*\rangle.
\end{align}
Comparing \eqref{eq:apply-two-coros-1} and \eqref{eq:apply-alg:intersect}, we have
\begin{align}\label{eq:relation}
    w^*\;=|_{s}\;\sigma_{\Vector{b}}\cdot\ta(w),
\end{align}
where the notion ``$=|_{s}$'' means ``is equal to up to multiplication by a nonzero rational number''. Suppose $v_1+c~(c\in\Q)$ is a factor of $w^*$.
By \eqref{eq:relation},
$v_1+c$ is a factor of $\sigma_{\Vector{b}}\cdot\ta(w)$.
So, by the choice for $\vA$ and $\vB$
there exists a linear factor $\xi$ of $w$, such that
$\sigma_{\vB}\cdot\tau_{\vA}(\xi)=|_s v_1+c$.
Also, recall that in \eqref{eq:factors}, $LFactor^{(1)}$ is a factor of $w$ and any linear factor of
$LFactor^{(1)}$ has the form
$\sum_{i\in T}u_i$, where $T$ is a nonempty subset of $\{1,\ldots,n\}$. For any $T\subset \{1, \ldots, n\}$, define
\begin{equation}
    \chi_{T}:=\;\begin{cases}
    1+\sum_{j\in T\setminus\{1\}}a_{j-1}, & {\rm if}\;1\in T,\\
    \sum_{j\in T}a_{j-1}, & {\rm if}\;1\notin T.
    \end{cases}
\end{equation}
First, we show that
if there exists a nonempty subset $T\subset \{1, \ldots, n\}$ such that $c=\frac{1}{\chi_{T}}\sum_{j\in T\setminus\{1\}}b_{j-1}$,
then $\ta(\sum_{i\in T}u_i)$ is a factor of $\ta(LFactor^{(1)})$.
In fact, by the definitions of $\ta$ and $\sigma_{\Vector{b}}$ (see \eqref{eq:def-linear-map} and \eqref{eq:def-sample-mapv}), 
\begin{align}\label{eq:alg-allone-correct-eq1}
  \sigma_{\Vector{b}}\cdot\ta(\sum_{i\in T}u_i)\;
  &=\;\sigma_{\Vector{b}}\big(\chi_{T}\cdot v_1+\sum_{j\in T\setminus\{1\}}v_j\big)\nonumber\\
  &=\;\chi_{T}\sigma_{\Vector{b}}\big(v_1+\frac{1}{\chi_{T}}\sum_{j\in T\setminus\{1\}}v_j\big)\nonumber\\
  &=\;\chi_{T}\big(v_1+\frac{1}{\chi_{T}}\sum_{j\in T\setminus\{1\}}b_{j-1}\big).
\end{align}
Here, remark that our choice for $\vA$ and $\vB$ guarantees that  $\chi_{T}\ne0$. So, if $c=\frac{1}{\chi_{T}}\sum_{j\in T\setminus\{1\}}b_{j-1}$, then
\begin{align}\label{eq:linear-factor}
    \sigma_{\Vector{b}}\cdot\ta(\sum_{i\in T}u_i)\;=|_{s}\;v_1+c,
\end{align}
Note that $\vB\notin V_{\vB}$. By the definition of $V_{\vB}$,
we know that $\sum_{i\in T}u_i$ is exactly the linear factor $\xi$ of $LFactor^{(1)}$ such that $\sigma_{\vB}\cdot\tau_{\vA}(\xi)=|_s v_1+c$.
Thus, $\ta(\sum_{i\in T}u_i)=\chi_{T}v_1+\sum_{j\in T\setminus\{1\}}v_j$ is  a factor of $\ta(LFactor^{(1)})$.
On the other hand, from the above discussion, we see that
any linear factor of $\ta(LFactor^{(1)})$ has the form $\ta(\sum_{i\in T}u_i)$, and
for any $T\subset \{1, \ldots, n\}$, if $\ta(\sum_{i\in T}u_i)$ is  a factor of $\ta(LFactor^{(1)})$, then there exists a
factor $v_1+c$ of $w^*$ such that
$c=\frac{1}{\chi_{T}}\sum_{j\in T\setminus\{1\}}b_{j-1}$.
Therefore, by going over the loop Line \ref{line:allone-5}--Line \ref{line:allone-8} in Algorithm \ref{alg:allone}, we can find all linear factors of $\ta(LFactor^{(1)})$.

Note that in practice, we cannot guarantee the input $\vA$ is generic or the vector $\vB$ we choose in
Algorithm \ref{alg:allone}--Line \ref{line:allone-1} is from $\Theta_{\vA}\setminus (B_{\vA}\cup V_{\vB})$, so Algorithm \ref{alg:allone} is probabilistic.
\end{proof}

\subsubsection{Correctness of Algorithm \ref{alg:degree}}\label{subsubsec:correct-degree}

For every $j\in\{1,\ldots,n\}$ and for any point
$\Vector{b}:=(b_1,\ldots,b_{n-1})\in {\mathbb Q}^{n-1}$,
we define the homomorphism $\sigma_{\vB}^{(j)}:\Q[\vV,\vX]\rightarrow\Q[v_j,\vX]$ such that
\begin{align}\label{eq:def-sample-map-j}
    \sigma_{\vB}^{(j)}(g(\vV,\vX))\;=\;g(b_1,\ldots,b_{j-1},v_j,b_{j},\ldots,b_{n-1},\vX).
\end{align}
Note that the homomorphism $\sigma_{\vB}^{(1)}$ is exactly $\sigma_{\vB}$ defined in \eqref{eq:def-sample-mapv}.
We also remark that some previous conclusions on $\sigma_{\vB}$ (for instance, Theorem \ref{thm:Hilbert} and Lemma \ref{lemma:alg-allone-pre}) also hold for $\sigma_{\vB}^{(j)}$.
In order to prove the correctness of Algorithm \ref{alg:degree}, we first prove
Lemma \ref{lemma:nonlinear-deg-upper}.
\begin{lemma}\label{lemma:nonlinear-deg-upper}
Let $\vf\subseteq\Q[\vU,\vX]$.
Suppose that $\mI(\mV(\vf)_{\infty})=\langle w\rangle$, where $w\in{\mathbb Q}[\Vector{u}]$.
Let $NLFactor$ be the product of all nonlinear irreducible factors of $w$.
Then, there exists an affine variety $A\subset\C^{n-1}$ such that for any $\vA\in\Q^{n-1}\setminus A$ and for every $j\in \{1,\ldots,n\}$,
there exist a Zariski dense subset $\Theta_{\vA,j}\subseteq\Q^{n-1}$ and an affine variety $B_{\vA,j}\subset\C^{n-1}$ such that for any $\Vector{b}\in \Theta_{\vA,j}\setminus B_{\vA,j}$,
\begin{align}\label{eq:lemma:nonlinear-deg-upper}
    \deg(\ta(NLFactor),v_j)\;\le\;\deg(\E_j, v_j)+(k_1-k_j),\;where
\end{align}
\begin{enumerate}[$(1)$]
    \item $\E_j$ is defined as the product of nonlinear irreducible factors of the polynomial $\sigma_{\vB}^{(j)}\cdot\ta(w)$, and
    \item 
    $k_j$ is defined as the number of nonlinear irreducible factors of the polynomial $\sigma_{\vB}^{(j)}\cdot\ta(w)$.
\end{enumerate}
In addition, for $j=1$, the equality in
\eqref{eq:lemma:nonlinear-deg-upper} holds.
\end{lemma}
\begin{proof}
By Lemma \ref{lemma:degree-correct-pre},  there exists an affine variety $A\subset\C^{n-1}$ such that for any $\vA\in\Q^{n-1}\setminus A$, 

 \begin{enumerate}
\item[(i)] for every irreducible factor $\xi$ of $w$, $\ta(\xi)$ is an irreducible factor of $\ta(w)$ satisfying $\deg(\ta(\xi),v_1)=\deg(\ta(\xi),\vV)=\deg(\xi,\vU)$. \label{item:lemma:nonlinear-deg-upper-1}
\end{enumerate}
In the argument below, we assume
$\vA\in {\mathbb Q}^{n-1}\setminus A$.
For every $j\in\{1,\ldots,n\}$,
by Theorem \ref{thm:Hilbert}, there exists a Zariski dense subset $\Theta_{\vA, j}\subseteq\Q^{n-1}$ and an
affine variety $B_{\vA, j}\subset {\mathbb C}^{n-1}$ such that for any $\Vector{b}\in \Theta_{\vA, j}\setminus B_{\vA, j}$,
\begin{enumerate}
\item[(ii)]for every irreducible factor $\xi$ of $w$ with $\deg(\ta(\xi), v_j)>0$,
$\sigma_{\vB}^{(j)}\cdot\ta(\xi)$ is an irreducible factor of $\sigma_{\vB}^{(j)}\cdot\ta(w)$ satisfying
$\deg(\sigma_{\vB}^{(j)}\cdot\ta(\xi),v_j)\;=\;\deg(\ta(\xi),v_j)$.\label{item:lemma:nonlinear-deg-upper-2}
\end{enumerate}
Notice that by (i),  for every irreducible factor $\xi$ of $w$, we have $\deg(\ta(\xi),v_1)=\deg(\xi,\vU)>0$.
So, by (ii),  for any $\vB\in \Theta_{\vA,1}\setminus B_{\vA, 1}$, we have
$\deg(\ta(NLFactor),v_1)\;=\;\deg(\E_1, v_1)$.
That means  the equality in  \eqref{eq:lemma:nonlinear-deg-upper} holds for $j=1$.
For every $j\in\{2,\ldots,n\}$, we define
\begin{align*}
    N_{j}&:=\;\{\ta(\xi)\mid\ta(\xi)\;{\rm is\;a\;nonlinear\;irreducible\;factor\;of\;}\ta(w)\;{\rm and}\;\deg(\ta(\xi),v_j)>1\},\\
    L_{j}&:=\;\{\ta(\xi)\mid\ta(\xi)\;{\rm is\;a\;nonlinear\;irreducible\;factor\;of\;}\ta(w)\;{\rm and}\;\deg(\ta(\xi),v_j)\le1\}.
\end{align*}
Below, for any set $S$, $\#(S)$ denotes the number of elements of the set $S$.
By (i) and (ii), for any $\vB\in \Theta_{\vA,1}\setminus B_{\vA, 1}$,
for every irreducible factor $\xi$ of $w$,
$\sigma_{\vB}^{(1)}\cdot\ta(\xi)$ is an irreducible factor of $\sigma_{\vB}^{(1)}\cdot\ta(w)$ satisfying
$\deg(\sigma_{\vB}^{(1)}\cdot\ta(\xi),v_1)\;=\;\deg(\ta(\xi),\vV)$.
So, the number of nonlinear irreducible factors of $\ta(\xi)$ is equal to that of
$\sigma_{\vB}^{(1)}\cdot\ta(\xi)$,
i.e.,
$\#(N_{j}\cup L_{j})=k_1$.
By (ii), we have
   $\#(N_{j})=k_j$ and $\deg(\E_j, v_j)=\sum_{\ta(\xi)\in N_{j}}\deg(\ta(\xi),v_j)$. So, since $N_{j}\cap L_{j}=\emptyset$, we have $\#(L_{j})=k_1-k_j$. By the definition of $NLFactor$ in \eqref{eq:factors} and by the fact  (i),
$N_{j}\cup L_{j}$ is also the set of all  irreducible factors of $\ta(NLFactor)$.
Therefore, 
\begin{align*}
\deg(\ta(NLFactor),v_j)\;
&=\;\sum_{\ta(\xi)\in N_{j}}\deg(\ta(\xi),v_j)+\sum_{\ta(\xi)\in L_{j}}\deg(\ta(\xi),v_j)\\
&\le\;\sum_{\ta(\xi)\in N_{j}}\deg(\ta(\xi),v_j)+\#(L_{j})\\
&=\;\sum_{\ta(\xi)\in N_{j}}\deg(\ta(\xi),v_j)+(k_1-k_j)\\
&=\;\deg(\E_j, v_j)+(k_1-k_j).
\end{align*}
\end{proof}

\begin{lemma}[Algorithm \ref{alg:degree}: Degrees]
\label{lemma:alg-degree-correct}
The probabilistic algorithm Algorithm \ref{alg:degree} is correct.
\end{lemma}
\begin{proof}
Suppose $\vf$ is the input system of likelihood equations
and $\mI(\mV(\vf)_{\infty})\;=\;\langle w\rangle$
(i.e., Assumption \ref{ap:principal} holds).
We prove that the $j$-th round of the loop (for $j=1,\ldots,n$) in Algorithm \ref{alg:degree}  is correct.
By  Lemma \ref{lemma:alg-allone-pre},
Theorem \ref{thm:Hilbert},
and Lemma \ref{lemma:nonlinear-deg-upper},
there exists an affine variety $A\subset\C^{n-1}$ such that for any $\vA\in\Q^{n-1}\setminus A$,
there exist a Zariski dense subset $\Theta_{\vA,j}\subseteq\Q^{n-1}$ and an affine variety $B_{\vA,j}\subset\C^{n-1}$ such that for any $\vB\in \Theta_{\vA,j}\setminus B_{\vA,j}$, we have 
\begin{align}
\mI(\mV(\sigma^{(j)}_{\Vector{b}}\cdot\ta(\vf))_{\infty})\;&=\;\langle \sigma^{(j)}                       _{\Vector{b}}\cdot\ta(w)\rangle,\;\text{and}\label{eq:pf-of-lemma:alg-degree-correct}\\
\deg(\ta(w),v_j)\;&=\;\deg(\sigma_{\vB}^{(j)}\cdot\ta(w), v_j),
\label{eq:pf-of-lemma:alg-degree-correct2}
\end{align}
and we also have the formula  \eqref{eq:lemma:nonlinear-deg-upper} holds (for $j=1$, the equality in  \eqref{eq:lemma:nonlinear-deg-upper} holds).
Using Algorithm \ref{alg:degree}, we can compute the monic polynomial $w_j^*\in\Q[v_j]$ (see Algorithm \ref{alg:degree}--Line \ref{line:degrees-3}) such that
\begin{align}\label{eq:apply-alg:intersect-2}
    \mI(\mV(\sigma^{(j)}_{\Vector{b}}\cdot\ta(\vf))_{\infty})\;=\;\langle w_j^*\rangle.
\end{align}
By \eqref{eq:pf-of-lemma:alg-degree-correct} and \eqref{eq:apply-alg:intersect-2}, we have
\begin{align}\label{eq:relation-2}
\sigma^{(j)}_{\Vector{b}}\cdot\ta(w)\;=|_{s}\;w_j^*,
\end{align}
where the notion ``$=|_{s}$'' means ``is equal to up to multiplication by a nonzero rational number''.
So, by \eqref{eq:pf-of-lemma:alg-degree-correct2}, the $j$-st entry of the list $wDegree$, i.e., $\deg(w_j^*, v_j)$,  
is exactly $\deg(\ta(w),v_j)$,
and by \eqref{eq:lemma:nonlinear-deg-upper},
the $j$-st entry of the list $NLFactorDegreeBound$ (see Algorithm \ref{alg:degree}-Line \ref{line:degrees-9}),
is an upper bound for $\deg(\ta(NLFactor),v_j)$
(for $j=1$, the first entry is
exactly $\deg(\ta(NLFactor),v_1)$).

Similarly to Algorithm \ref{alg:allone}, in practice, we cannot guarantee the input $\vA$ is generic or the vector $\vB$ we choose in
Algorithm \ref{alg:degree}--Line \ref{line:degrees-2} is from $\Theta_{\vA, j}\setminus B_{\vA, j}$, so Algorithm \ref{alg:degree} is probabilistic.
\end{proof}
\begin{remark}
In \cite[Algorithm 2]{Tang2017}, the authors also used
the specialization/linear-lifting method for interpolating
the discriminants of likelihood equations.
In  the first step of \cite[Algorithm 2]{Tang2017}, one needs to compute  the  degrees of discriminants. Here, in Algorithm \ref{alg:degree}, we can only compute the upper bounds for the degrees of $\ta(NLFactor)$ {\it w.r.t} $v_i$ ($i\in \{2, \ldots, n\}$). It might affect the efficiency of the interpolation since one needs to do more sampling steps if these upper bounds are not tight, but it does not affect the correctness.  However, our experiments show that the
upper bounds computed by Algorithm \ref{alg:degree} are usually tight in practice.
\end{remark}

\subsubsection{Correctness of Algorithm \ref{alg:nonlinear}}\label{subsubsec:correct-nonlinear}

\vspace{-1.6mm}

\begin{lemma}[Algorithm \ref{alg:nonlinear}: Interpolation]
\label{lemma:alg-nonlinear-correct}
The probabilistic algorithm Algorithm
\ref{alg:nonlinear} is correct.
\end{lemma}
\begin{proof}
Suppose $\vf$ is a given system of likelihood equations
and $\mI(\mV(\vf)_{\infty})\;=\;\langle w\rangle$
(i.e., Assumption \ref{ap:principal} holds). By Lemma \ref{lemma:alg-allone-pre} and
Theorem \ref{thm:Hilbert}, there exists an affine variety $A\subset\C^{n-1}$ such that for any $\vA\in A$, there exists a Zariski dense subset $\Theta_{\vA}\subseteq\Q^{n-1}$ and an affine variety $B_{\vA}\subset\C^{n-1}$ such that for any $\vB\in\Theta_{\vA}\setminus B_{\vA}$, we have
\begin{enumerate}[(i)]
\item\label{item:lemma:alg-nonlinear-correct-1}
$\mI(\mV(\sigma_{\Vector{b}}\cdot\ta(\vf))_{\infty})\;=\;\langle \sigma_{\Vector{b}}\cdot\ta(w)\rangle$,
\item\label{item:lemma:alg-nonlinear-correct-4}
$\sigma_{\vB}\cdot\ta(NLFactor)=|_{s}$ the product of nonlinear irreducible factors of $\sigma_{\vB}\cdot\ta(w)$, and
\item\label{item:lemma:alg-nonlinear-correct-5}
$\sigma_{\vB}\cdot\ta(LFactor^{(2)})=|_{s}$  the product of linear factors of $\sigma_{\vB}\cdot\ta(w/LFactor^{(1)})$,
\end{enumerate}
where the notion ``$=|_{s}$'' means ``is equal to up to multiplication by a nonzero rational number''.
In the argument below, we assume
$\vA\in {\mathbb Q}^{n-1}\setminus A$.
Notice that when we run this algorithm,
we also input the polynomial  $\ta(LFactor^{(1)})$ (defined in
\eqref{eq:lcfactors}),
a list $wDegree$ (whose
	   $j$-th entry is $\deg(\ta(w),v_j)$) and
a list $NLFactorDegreeBound$ (whose
	    first entry is $\deg(\ta(NLFactor),v_1)$ and
	    $j$-th entry is an upper bound of $\deg(\ta(NLFactor),v_j)$ for $j=2,\ldots,n$). Remark that these can be computed by the previous Algorithm
\ref{alg:allone} and Algorithm \ref{alg:degree}.


First, we explain how to interpolate $\ta(NLFactor)$.
Assume that $$\deg(\ta(NLFactor), v_1)=d_1.$$
Recall again that $d_1$ is exactly the first entry in the input list $NLFactorDegreeBound$.
Then, we can write
$$\ta(NLFactor)=v_1^{d_1}+\sum_{i=1}^{d_1}C_{i}\cdot v_1^{d_1-i},$$
where $C_{i}\in {\mathbb Q}[v_2, \ldots, v_n]$ is a homogeneous polynomial with total degree $\deg(C_{i})=i$ (here, recall that by Remark \ref{rmk:homo}, $NLFactor$ is homogeneous and so is $C_i$).
In Algorithm \ref{alg:nonlinear}, we interpolate
$C_{i}$ by a standard evaluation/linear-lifting method.
More specifically, $C_{i}$ is a linear combination of
the monomials in the following set
\begin{align*}
{\rm Mon}^{(NL)}_i\;:=\;&\{U\in {\mathbb Q}[v_2, \ldots, v_n]\mid\deg(U)=i, \;\text{and} \\
&\deg(U,v_j)\le NLFactorDegreeBound[j]\}.
\end{align*}
Here, recall that $NLFactorDegreeBound[j]$ is the input of this algorithm, which gives an upper bound of $\deg(\ta(NLFactor),v_j)$.
For $i\in \{1, \ldots, d_1\}$, suppose the number of
monomials in ${\rm Mon}^{(NL)}_i$ is $N_i$, and suppose
$N:=\max(N_1, \ldots, N_{d_1})$. Then, we can compute all coefficients
of $C_{i}$ by sampling
$N$ times. By ``sampling once", we mean choosing a vector
$\vB\in \Theta_{\vA}\setminus B_{\vA}$, and  computing
the monic polynomial $w^*\in\Q[v_1]$ (see Algorithm \ref{alg:nonlinear}--Line \ref{line:nonlinear-7}) such that
\begin{align}\label{eq:apply-alg:intersect-3}
\mI(\mV(\sigma_{\Vector{b}}\cdot\ta(\vf))_{\infty})\;=\;\langle w^*\rangle.
\end{align}
Comparing \eqref{item:lemma:alg-nonlinear-correct-1} and \eqref{eq:apply-alg:intersect-3}, we have
\begin{align}\label{eq:relation-3}
    \sigma_{\Vector{b}}\cdot\ta(w)\;=|_{s}\;w^*.
\end{align}
So, by \eqref{item:lemma:alg-nonlinear-correct-4},  $\sigma_{\vB}\cdot\ta(NLFactor)$ is equal to the product of nonlinear irreducible factors of $w^*$ (i.e., $\eta^*$ in Algorithm \ref{alg:nonlinear}--Line \ref{line:nonlinear-8}) up to scaling.
Therefore, we choose different vectors from $\Theta_{\vA}\setminus B_{\vA}$, say $N$ times, and we can recover all  the coefficients of
$C_{i}$ by solving linear systems (see Algorithm \ref{alg:nonlinear}--Line \ref{line:nonlinear-12}).

Second, we similarly interpolate $\ta(LFactor^{(2)})$.
Assume that $$\deg(\ta(LFactor^{(2)}), v_1)=d_2.$$
Notice that by \eqref{eq:lcfactors} we can compute $d_2$ by the fact that $$d_2= \deg(\ta(w), v_1)-\deg(\ta(LFactor^{(1)}),v_1)-\deg(\ta(NLFactor),v_1)
,$$ where $\deg(\ta(w), v_1)$ is the first entry in the input list $wDegree$,
$\deg(\ta(LFactor^{(1)}),v_1)$ can be read directly from the input $\ta(LFactor^{(1)})$, and $\deg(\ta(NLFactor),v_1)$ is exactly $d_1$. Then, we can write
$$\ta(LFactor^{(2)})=v_1^{d_2}+\sum_{i=1}^{d_2}D_{i}\cdot v_1^{d_2-i},$$
where $D_{i}\in {\mathbb Q}[v_2, \ldots, v_n]$ is a homogeneous polynomial with total degree $\deg(D_{i})=i$.
Note that $D_{i}$ is a linear combination of
the monomials in
$${\rm Mon}^{(L)}_i:=\{U\in {\mathbb Q}[v_2, \ldots, v_n]\mid\deg(U)=i, \deg(U,v_j)\le \deg(\ta(LFactor^{(2)}),v_j)
,$$
where
{\footnotesize
$$\deg(\ta(LFactor^{(2)}),v_j)=\deg(\ta(w), v_j)-\deg(\ta(LFactor^{(1)}),v_j)-\deg(\ta(NLFactor),v_j).$$
}For $i\in \{1, \ldots, d_2\}$, suppose the number of
monomials in ${\rm Mon}^{(L)}_i$ is $M_i$, and suppose
$M:=\max(M_1, \ldots, M_{d_2})$. Then, we can compute all coefficients
of $D_{i}$ by sampling
$M$ times.
Again, by ``sampling once", we mean choosing a vector
$\vB\in \Theta_{\vA}\setminus B_{\vA}$, and  computing
the monic polynomial $w^*\in\Q[v_1]$ such that
\eqref{eq:apply-alg:intersect-3} holds.
By \eqref{eq:relation-3} and \eqref{item:lemma:alg-nonlinear-correct-5}, $\sigma_{\vB}\cdot\ta(LFactor^{(2)})$ is equal to the product of nonlinear irreducible factors  of $w^{*}/\sigma_{\vB}\cdot\ta(LFactor^{(1)})$ up to scaling.
So, we choose different vectors from $\Theta_{\vA}\setminus B_{\vA}$, say $M$ times, and we can recover all  the coefficients of
$D_{i}$ by solving linear systems  (see Algorithm \ref{alg:nonlinear}--Line \ref{line:nonlinear-28}).

Similarly to Algorithm \ref{alg:allone} and Algorithm \ref{alg:degree}, in practice, we cannot guarantee the input $\vA$ is generic or the vector $\vB$ we choose in
Algorithm \ref{alg:nonlinear}--Line \ref{line:nonlinear-6} (or Line
\ref{line:nonlinear-22})  is from $\Theta_{\vA}\setminus B_{\vA}$, so Algorithm \ref{alg:nonlinear} is probabilistic.
\end{proof}

\begin{remark}
Notice that in Algorithm
\ref{alg:nonlinear},
we interpolate the two output polynomials
$\ta(NLFactor)$ and $\ta(LFactor^{(2)})$ simultaneously.
Suppose we need to do the sampling steps respectively $N$ and $M$ times for interpolating these two polynomials. In Algorithm
\ref{alg:nonlinear}, we only
carry out the sampling steps  $\max(M,N)$ times.
\end{remark}

\subsection{Running Example}\label{subsec:examples}
We revisit the $3\times 3$ symmetric matrix model in Example \ref{ex:symmetric}.
Recall that
 the system of Lagrange likelihood equations can be written as $\vf=\{f_1,\ldots, f_8\}$ in \eqref{eq:ex1}.
There are $6$ parameters $u_1,\ldots, u_6$ and $8$ variables $p_1,\ldots, p_6,\lambda_1,\lambda_2$. The goal of   this section  is to compute a polynomial $w\in\Q[\vU]$ by Algorithm \ref{alg:new} such that ${\mathcal I}({\mathcal V}(\vf)_{\infty})=\langle w\rangle$. Recall that we assume $w
$ can be factorized as in \eqref{eq:factors}:
\begin{align*}
w\;=\;LFactor^{(1)}\cdot NLFactor\cdot LFactor^{(2)}.
\end{align*}
In the first example below, we apply a linear transformation to the original parameters.
\begin{example}[{\bf LinearTransformation}]\label{eg:linearop}
We apply the following linear transformation $\ta$, where $\vA=(a_1,\ldots,a_5)=(1,\ldots,1)\in\Q^5$, to the system $\vf$ \eqref{eq:ex1}:
\begin{align}
u_1&=v_1,      &   u_2&=v_2+v_1,  & u_3&=v_3+v_1, \notag\\
u_4&=v_4+v_1,  & u_5&=v_5+v_1,&u_6&=v_6+v_1.\label{eq:exlinear}
\end{align}
Then, we obtain the system after the linear transformation $\ta(\vf)$:
\begin{align}\label{eq:newf}
\begin{array}{rl}
f_1&=\;p_1(\lambda_1+(2p_3p_5-4p_1p_6)\lambda_2)-v_1, \\
f_2&=\;p_2(\lambda_1+(8p_4p_6-2p_5^2)\lambda_2)-(v_2+v_1), \\
f_3&=\;p_3(\lambda_1+(2p_1p_5-4p_3p_4)\lambda_2)-(v_3+v_1), \\
f_4&=\;p_4(\lambda_1+(8p_2p_6-2p_3^2)\lambda_2)-(v_4+v_1), \\
f_5&=\;p_5(\lambda_1+(2p_1p_3-4p_2p_5)\lambda_2)-(v_5+v_1), \\
f_6&=\;p_6(\lambda_1+(8p_2p_4-2p_1^2)\lambda_2)-(v_6+v_1), \\
f_7&=\;2p_1p_3p_5+8p_2p_4p_6-2p_2p_5^2-2p_3^2p_4-2p_1^2p_6, \\
f_8&=\;p_1+p_2+p_3+p_4+p_5+p_6-1.
\end{array}
\end{align}
Recall that if we apply the transformation to \eqref{eq:factors}, then   $\ta(w)$ can be factored as in
\eqref{eq:lcfactors}:
\begin{align*}
\ta(w)\;=\;\ta(LFactor^{(1)})\cdot\ta(NLFactor)\cdot\ta(LFactor^{(2)}).
\end{align*}
In the following examples, we will compute the above three factors respectively.
\end{example}

\begin{example}[{\bf AllOne}]\label{eg:allone}
We show how to compute the factor $\ta(LFactor^{(1)})$ in \eqref{eq:lcfactors} by  Algorithm \ref{alg:allone}.
First, we choose $\Vector{b}=(b_1,\ldots,b_5)=(29,43,89,149,247)\in\Q^{5}$.
By specializing $(v_2, \ldots, v_6)$ with $\Vector{b}$ in $\ta(\vf)$ shown in \eqref{eq:newf}, we obtain a specialized system $\sigma_{\Vector{b}}\cdot\ta(\vf)\subset {\mathbb Q}[v_1, \vX]$.
According to Algorithm \ref{alg:standard}, we compute the monic polynomial $w^*\in {\mathbb Q}[v_1]$ such that \[{\mathcal I}({\mathcal V}(\sigma_{\Vector{b}}\cdot\ta(\vf))_{\infty})=\langle w^*\rangle. \]
From the computational result, we list all monic linear factors of $w^*$:
{\footnotesize
\begin{align}\label{eq:lfactorlist}
[v_1+\frac{118}{3},\; v_1+\frac{319}{3},\; v_1+\frac{485}{3},\; v_1+\frac{557}{6},\; v_1+\frac{343}{2},\; v_1+\frac{101}{4},\; v_1+\frac{327}{4}].
\end{align}}For the first factor $v_1+\frac{118}{3}$ in the above list \eqref{eq:lfactorlist}, note that the constant term is $\frac{118}{3}$, which is exactly the value ``$c$"   in Algorithm \ref{alg:allone}--Line \ref{line:allone-5}.
For the subset $T=\{1,2,4\}\subset \{1,\ldots,6\}$, the value
``$\chi_T$" defined in Algorithm \ref{alg:allone}--Line \ref{line:allone-9} is equal to $1+a_1+a_3=3$ (here, recall that $a_1=a_3=1$ according to Example \ref{eg:linearop}).
So, it is straightforward to compute that $\frac{1}{\chi_T}\sum_{j\in T\setminus\{1\}}b_{j-1}=\frac{118}{3}$.
And hence, we have $c=\frac{1}{\chi_T}\sum_{j\in T\setminus\{1\}}b_{j-1}$, which means the condition in Algorithm \ref{alg:allone}--Line \ref{line:allone-7} holds.
Therefore, by Algorithm \ref{alg:allone}--Line\ref{line:allone-8}, we known that  $\ta(LFactor^{(1)})$ has a factor $\chi_{T}v_1+\sum_{j\in T\setminus\{1\}}v_j=3v_1+v_2+v_4$.
Similarly, by the second, the third  and the fourth  factors in the list \eqref{eq:lfactorlist},
we can compute that $3v_1+v_2+v_3+v_6$, $3v_1+v_4+v_5+v_6$ and $6v_1+v_2+v_3+v_4+v_5+v_6$ are factors of $\ta(LFactor^{(1)})$.
It is also straightforward to check that the last three factors in the list \eqref{eq:lfactorlist} do not satisfy the condition in Algorithm \ref{alg:allone}--Line \ref{line:allone-7}.
So, we conclude that
{\footnotesize
\begin{align}
  \ta(LFactor^{(1)})\;=\;&(3v_1+v_2+v_4)\cdot(3v_1+v_2+v_3+v_6)\cdot(3v_1+v_4+v_5+v_6)\notag\\
  &\cdot(6v_1+v_2+v_3+v_4+v_5+v_6).\label{eq:LinearAllOne}
\end{align}}
\end{example}

\begin{example}[{\bf Degrees}]\label{eg:degrees}
In this example,
we show how to compute by Algorithm \ref{alg:degree}
\begin{itemize}
    \item $\deg(\ta(w),v_i)$ for every $i\in\{1,\ldots,n\}$ (recorded in a list $wDegree$),
    \item $\deg(\ta(NLFactor),v_1)$ and an upper bound of  $\deg(\ta(NLFactor),v_j)$ for every $j\in\{2,\ldots,n\}$
    (recorded in a list $NLFactorDegreeBound$).
\end{itemize}
Later in Example \ref{eg:nonlinear}, we will use these results to interpolate $\ta(NLFactor)$ and $\ta(LFactor^{(2)})$.

First, we compute $\deg(\ta(w), v_1)$ and $\deg(\ta(NLFactor), v_1)$.
We choose a 
vector
$\Vector{b}=(4,3,7,8,9)\in\Q^5$.
By specializing $(v_2, \ldots, v_6)$ with $\Vector{b}$ in $\ta(\vf)$ shown in \eqref{eq:newf}, we obtain a specialized system $\sigma_{\Vector{b}}^{(1)}\cdot\ta(\vf)\subset {\mathbb Q}[v_1, \vX]$ (here, recall that $\sigma_{\Vector{b}}^{(j)}$ is defined in \eqref{eq:def-sample-map-j}, where $\sigma_{\Vector{b}}^{(1)}$ is exactly $\sigma_{\Vector{b}}$).
By Algorithm \ref{alg:standard}, we compute the monic polynomial $w^*\in {\mathbb Q}[v_1]$ such that ${\mathcal I}({\mathcal V}(\sigma_{\Vector{b}}^{(1)}\cdot\ta(\vf))_{\infty})=\langle w^*\rangle$.
By the computational result, we obtain all factors of $w^*$:
{\footnotesize
\begin{align*}
[v_1^3+\frac{49}{2}v_1^2+\frac{353}{2}v_1+\frac{689}{2},\; v_1+\frac{11}{3},\; v_1+\frac{29}{4},\; v_1+\frac{11}{4},\; v_1+\frac{16}{3},\; v_1+\frac{31}{6},\; v_1+8,\; v_1+\frac{11}{2}],
\end{align*}}and by Remark \ref{rmk:homo},  each factor in the above
list has multiplicity $1$.
By \eqref{eq:apply-two-coros-1}, we have
\begin{align}
    \sigma_{\Vector{b}}^{(1)}\cdot\ta(w)\;=|_{s}\;w^*,
\end{align}
where recall that the notion ``$=|_{s}$'' means ``is equal to up to multiplication by a nonzero rational number''.
So, by Theorem \ref{thm:Hilbert}, we have $\deg(\ta(w),v_1)=\deg(\sigma_{\Vector{b}}^{(1)}\cdot\ta(w))=10$.
Notice that  the first factor is the only nonlinear factor in the above list.
Then, by Lines \ref{line:degrees-5}--\ref{line:degrees-9} in Algorithm \ref{alg:degree},
we have $\deg(\ta(NLFactor), v_1)=\deg(v_1^3+\frac{49}{2}v_1^2+\frac{353}{2}v_1+\frac{689}{2})=3$.

Second, we compute $\deg(\ta(w), v_2)$ and an upper bound of $\deg(\ta(NLFactor), v_2)$.
Here, we choose $\Vector{b}=(2,4,7,8,1)\in\Q^5$. By specializing $(v_1, v_3,v_4, v_5, v_6)$ with $\Vector{b}$ in $\ta(\vf)$ shown in \eqref{eq:newf}, we obtain a specialized system $\sigma_{\Vector{b}}^{(2)}\cdot\ta(\vf)\subset {\mathbb Q}[v_2, \vX]$.
Similarly, we can compute the list of all factors of
$\sigma_{\Vector{b}}^{(2)}\cdot\ta(w)$:
{\footnotesize
\begin{align}\label{eq:lidegreev2}
[v_2+32, v_2-25, v_2+11, v_2+6, v_2+13],
\end{align}}and every factor has multiplicity $1$.
Therefore, by Theorem \ref{thm:Hilbert}, $\deg(\ta(w), v_2)=\deg(\sigma_{\Vector{b}}^{(2)}\cdot\ta(w))=5$.
Note that there is no nonlinear factors in the above list \eqref{eq:lidegreev2}.
So, by Lines \ref{line:degrees-5}--\ref{line:degrees-9} in Algorithm \ref{alg:degree}, we have $\deg(\ta(NLFactor), v_2)\le 1$ (notice that
in this example, the values ``$k_1$" and ``$k_2$" defined in Algorithm \ref{alg:degree}-Line \ref{line:degrees-7} are $1$ and $0$, respectively). For any $i$ $(3\le i\le 6)$, by the same method, we can compute $\deg(\ta(w), v_i)$ and an upper bound of $\deg(\ta(NLFactor), v_i)$. Finally, we obtain $wDegree=[10, 5, 6, 5, 6, 5]$ and
$NLFactorDegreeBound=[3, 1, 2, 1, 2, 1]$.
\end{example}

\begin{example}[{\bf Interpolation}]\label{eg:nonlinear}
We apply Algorithm \ref{alg:nonlinear} to compute
$\ta(NLFactor)$ and $\ta(LFactor^{(2)})$ in \eqref{eq:lcfactors}.
From Example \ref{eg:degrees}, we know that $\deg(\ta(w),v_1)=10$ and
$\deg(\ta(NLFactor),v_1)=3$.
By \eqref{eq:LinearAllOne}, we see that
$\deg(\ta(LFactor^{(1)}),v_1)=4$.
So, by \eqref{eq:lcfactors}, we have
$\deg(\ta(LFactor^{(2)}),v_1)=3$.
By Lemma \ref{lemma:degree-correct-pre}, we can assume that
\begin{align}
\ta(NLFactor)\;&=|_{s}\;v_1^{3}+C_1\cdot v_1^2+C_2\cdot v_1+C_3,\;{\rm and}\label{eq:ex-NLFactor-up-to-scaling}\\
\ta(LFactor^{(2)})\;&=|_{s}\;v_1^{3}+D_1\cdot v_1^2+D_2\cdot v_1+D_3,\label{eq:ex-LFactor2-up-to-scaling}
\end{align}
where $C_{i}$ or $D_i$ in $\Q[v_2,\ldots,v_6]$ is either the zero polynomial or  homogeneous with total degree $i$, and the notion ``$=|_{s}$'' means ``is equal to up to multiplication by a nonzero rational number''.

First, we show how to interpolate $\ta(NLFactor)$.
By the upper bound of $\deg(\ta(NLFactor),v_j)$ $(j=2,\ldots,6)$ computed in Example \ref{eg:degrees}, we
enumerate all the
possible monomials for each $C_i$ ($i\in \{1, 2, 3\}$) in \eqref{eq:ex-NLFactor-up-to-scaling}, and we record these monomials in ${\tt Mon}_i$:
{\footnotesize
\begin{align*}
    {\tt Mon}_1\;=\;\{&v_2,v_3,v_4,v_5,v_6\},\\
    {\tt Mon}_2\;=\;\{&v_2v_3, v_2v_4, v_2v_6, v_2v_5, v_3^2, v_3v_4, v_3v_6, v_3v_5, v_4v_6, v_4v_5, v_5v_6, v_5^2\},\\
    {\tt Mon}_3\;=\;\{&v_2v_3^2, v_2v_3v_4, v_2v_3v_6, v_2v_3v_5, v_2v_4v_6, v_2v_4v_5, v_2v_5v_6, v_2v_5^2, v_3^2v_4, v_3^2v_6, v_3^2v_5, v_3v_4v_6, \\
    &v_3v_4v_5, v_3v_5v_6, v_3v_5^2, v_4v_5v_6, v_4v_5^2, v_5^2v_6\}.
\end{align*}}Below, we show how to interpolate $C_1$ by a sampling/linear-lifting method. One can similarly compute $C_2$ and $C_3$.
By the monomials recorded in ${\tt Mon}_1$, we can assume that $$C_1=C_{11}v_2+C_{12}v_3+C_{13}v_4+C_{14}v_5+C_{15}v_6, \;\text{where}\; C_{1i}\in\mathbb{Q}\;\text{for}\; i\in\{1,\ldots,5\}.$$
For the first time of sampling, we choose one vector $\vB_1=(27,17,8,5,26)$, and we compute the monic polynomial $w^*\in\Q[v_1]$ such that ${\mathcal I}({\mathcal V}(\sigma_{\Vector{b}_1}\cdot\ta(\vf))_{\infty})=\langle w^*\rangle$. From the computational result, we obtain
that $w^*$ has the list of factors:
{\footnotesize
\begin{align}\label{eq:listfactorw}
    [v_1^3+\frac{161}{2}v_1^2+\frac{3733}{2}v_1+\frac{19477}{2},\; v_1+\frac{71}{4},\;v_1+13,\;v_1+\frac{21}{4},\;v_1+\frac{83}{6},\;v_1+\frac{37}{2},\;v_1+\frac{70}{3},\;v_1+\frac{35}{3}],
\end{align}}where each factor has multiplicity $1$.
By Lemma \ref{lemma:alg-nonlinear-correct},
{\footnotesize
\begin{align}\label{eq:nonlinear1}
\sigma_{\vB_1}\cdot\ta(NLFactor)\;=|_{s}\;v_1^3+\frac{161}{2}v_1^2+\frac{3733}{2}v_1+\frac{19477}{2}.
\end{align}}Next, we respectively take $\vB_2=(9, 2, 18, 20, 24)$, $\vB_3=(11, 19, 25, 13, 29)$, $\vB_4=(28, 3, 14, 6, 30)$ and $\vB_5=(15, 4, 10, 12, 16)$.
Similarly, we can compute that
{\footnotesize
\begin{align}\label{eq:nonlinear2}
\sigma_{\vB_2}\cdot\ta(NLFactor)&\;=|_{s}\;v_1^3+\frac{131}{2}v_1^2+1222v_1+5940, \notag\\
\sigma_{\vB_3}\cdot\ta(NLFactor)&\;=|_{s}\;v_1^3+\frac{163}{2}v_1^2+\frac{3757}{2}v_1+10508, \notag\\
\sigma_{\vB_4}\cdot\ta(NLFactor)&\;=|_{s}\;v_1^3+\frac{207}{2}v_1^2+\frac{6161}{2}v_1+22953,\notag\\
\sigma_{\vB_5}\cdot\ta(NLFactor)&\;=|_{s}\;v_1^3+\frac{107}{2}v_1^2+824v_1+3640.
\end{align}}So far, we have done the sampling $5$ times since $\#({\tt Mon}_1)=5$.
Comparing the coefficients in  \eqref{eq:ex-NLFactor-up-to-scaling}, \eqref{eq:nonlinear1}, and \eqref{eq:nonlinear2}, we have the linear equations below
{\footnotesize
\begin{align*}
27C_{11}+17C_{12}+8C_{13}+5C_{14}+26C_{15}&\;=\;\frac{161}{2},\\
9C_{11}+2C_{12}+18C_{13}+20C_{14}+24C_{15}&\;=\;\frac{131}{2},\\
11C_{11}+19C_{12}+25C_{13}+13C_{14}+29C_{15}&\;=\;\frac{163}{2},\\
28C_{11}+3C_{12}+14C_{13}+6C_{14}+30C_{15}&\;=\;\frac{207}{2},\\
15C_{11}+4C_{12}+10C_{13}+12C_{14}+16C_{15}&\;=\;\frac{107}{2}.
\end{align*}}Solving the linear system yields $C_{11}=\frac{3}{2}, C_{12}=-\frac{1}{2}, C_{13}=\frac{3}{2}, C_{14}=-\frac{1}{2}$, and $C_{15}=\frac{3}{2}$. We can similarly compute $C_2$ and $C_3$, and we get
{\footnotesize
\begin{align}
\ta(NLFactor)\;=|_{s}\;&v_1^3+\frac{3}{2}v_1^2v_2-\frac{1}{2}v_1^2v_3+\frac{3}{2}v_1^2v_4-\frac{1}{2}v_1^2v_5+\frac{3}{2}v_1^2v_6+2v_1v_2v_4-v_1v_2v_5\notag\\
&+2v_1v_2v_6-\frac{1}{2}v_1v_3^2-v_1v_3v_4+\frac{1}{2}v_1v_3v_5
+2v_1v_4v_6-\frac{1}{2}v_1v_5^2\notag\\
&+2v_2v_4v_6-\frac{1}{2}v_2v_5^2-\frac{1}{2}v_3^2v_4.
\label{eq:nonlinear}
\end{align}}Notice that $\max(\#({\tt Mon}_1),\#({\tt Mon}_2),\#({\tt Mon}_3))=18$. That means we need to do the sampling
$18$ times for computing all $C_i$.

Second, we explain how to simultaneously  interpolate  $\ta(LFactor^{(2)})$.
By the previous results $\ta(LFactor^{(1)})$ \eqref{eq:LinearAllOne} computed in Example \ref{eg:allone},
$\deg(\ta(w), v_i)$  computed in Example \ref{eg:degrees}, and $\ta(NLFactor)$  \eqref{eq:nonlinear} computed before,
we can get $\deg(\ta(LFactor^{(2)}),v_{i})$ for $i=1,\ldots,6$.
So, we can enumerate all possible monomials of $D_1$, $D_2$ and $D_3$ in \eqref{eq:ex-LFactor2-up-to-scaling}.
Notice that  for interpolating $\ta(LFactor^{(2)})$, we can make use of the samplings we have done in the first step while interpolating $\ta(NLFactor)$. For instance,
for the first time of sampling we did,  by \eqref{eq:LinearAllOne}, we have
{\footnotesize
\begin{align}
    \sigma_{\vB_1}\cdot\ta(LFactor^{(1)})\;=\;54\cdot(v_1+\frac{35}{3})\cdot(v_1+\frac{70}{3})\cdot(v_1+13)\cdot(v_1+\frac{83}{6}).
\end{align}}So, by \eqref{eq:listfactorw}, \eqref{eq:nonlinear1}, and the fact that
{\footnotesize
\begin{align*}
    w^*\;=|_{s}\;\sigma_{\vB_1}\cdot\ta(w)= \sigma_{\vB_1}\cdot\ta(LFactor^{(1)}\cdot NLFactor\cdot LFactor^{(2)}),
\end{align*}}we have
{\footnotesize
\begin{align*}
    \sigma_{\vB_1}\cdot\ta(LFactor^{(2)})\;=\;(v_1+\frac{71}{4})\cdot(v_1+\frac{21}{4})\cdot(v_1+\frac{37}{2}).
\end{align*}}We remark that in this example, we need to sample at most $18$ times for computing all $D_i$.
But we have already sampled $18$ times while computing $\ta(NLFactor)$.
So, we only need to do the linear lifting, and we obtain
{\footnotesize
\begin{align}\label{eq:LinearNotAllOne}
\ta(LFactor^{(2)})\;=|_{s}\;&\frac{1}{64}\cdot(4v_1+2v_2+v_3)\cdot(4v_1+2v_4+v_5)\cdot(4v_1+v_3+v_5+2v_6).
\end{align}}

Finally, by the results we computed in the above examples, we apply the inverse of the linear transformation \eqref{eq:exlinear} to
$\ta(LFactor^{(1)})$ in \eqref{eq:LinearAllOne}, $\ta(NLFactor)$ in \eqref{eq:nonlinear} and $\ta(LFactor^{(2)})$ in \eqref{eq:LinearNotAllOne}, and we further  multiply the resulting polynomial by a proper integer. Then,
we get a generator polynomial $w$ of the ideal ${\mathcal I}({\mathcal V}(\vf)_{\infty})$:
{\footnotesize
\begin{align}
    w \;=\;&LFactor^{(1)}\cdot NLFactor\cdot LFactor^{(2)}\notag\\
      \;=\;&(u_1+u_2+u_4)\cdot(u_2+u_3+u_6)\notag\\
      &\cdot(u_4+u_5+u_6)\cdot(u_1+u_2+u_3+u_4+u_5+u_6)\notag\\
      &\cdot(u_1^2u_6-u_1u_3u_5-4u_2u_4u_6+u_2u_5^2+u_3^2u_4)\notag\\
      &\cdot(u_1+2u_2+u_3)\cdot(u_1+2u_4+u_5)\cdot(u_3+u_5+2u_6).\label{eq:exfactor}
\end{align}}Since the method is probabilistic, we have applied Algorithm \ref{alg:new} several times with different choices of  $\Vector{a}$ shown in Example \ref{eg:linearop}, and the results are the same as \eqref{eq:exfactor}, which is exactly \eqref{eq:symmetric} presented before.
\end{example}

\section{Implementation}\label{sec:implementation}

\subsection{Implementation}\label{subsec:implementation}
 {\tt Maple} code, testing models, and computational results are available online via:
{\footnotesize
	 \noindent\struc{\url{https://github.com/zhao-tq/nonproperness}}.
}
\begin{description}
	\item[Software] We implemented Algorithm \ref{alg:new} with {\tt Maple2023}, where we use the  {\tt FGb} command {\tt fgb\_gbasis}
for computing reduced Gr\"obner bases whenever we need to run Algorithm \ref{alg:standard}-Line \ref{algline:GRLGrobner}.
\end{description}
\begin{description}
\item[Hardware and System] We used a 2.3 GHz Intel Core i7 processor (16 GB of RAM) under Ubuntu 18.04.5.
\item[Testing Models] Testing Models are chosen from the literatures \citep*{DSS2009, SAB2005} (see Appendix \ref{sec:appendix}). 
\end{description}

\subsection{Computing nonproperness sets}\label{subsec:experiment}
 All testing models have been tested by the standard method (i.e. Algorithm \ref{alg:standard}), the interpolation method {\cite[Strategies 1--2]{Tang2017}} and the new method  (i.e., Algorithm \ref{alg:new}).
Table \ref{table:literatureOld} (Recall  Section \ref{sec:intro}) compares the timings of these methods.

\smallskip

\noindent
{\bf Instruction  for Table \ref{table:literatureOld}:}
\begin{description}
\item[(1)] For each testing model, the column ``MLD'' gives the  ML-degree $N$.
\item[(2)] In the column ``Timings--Standard'', we record the timings consumed by the standard method (Algorithm \ref{alg:standard}),
``OOM'' means ``out of memory".
\item[(3)] We record timings of the old interpolation methods {\cite[Strategies 1--2]{Tang2017}} in the columns ``Timings--Interpolation I'' and ``Timings--Interpolation II'', where
``OOT'' means  ``the computation does not finish in  $3$ days".
\item[(4)] We record the  timings consumed by the new method (Algorithm \ref{alg:new}) in the column ``Timings--New Method''.
\end{description}

\smallskip

\noindent
{\bf Conclusion from Table  \ref{table:literatureOld}:}
The new method (Algorithm \ref{alg:new}) indeed improves the efficiency significantly.
For smaller models with ML-degree less than $5$, standard method (Algorithm \ref{alg:standard}) finishes successfully within 20 seconds, interpolation method {\cite[page 354, Algorithm 2 (Strategy 2)]{Tang2017}} finishes within 500 seconds, while Algorithm \ref{alg:new} almost takes no time. For larger models with ML-degree greater than $5$, only interpolation method and Algorithm \ref{alg:new} can get results, and Algorithm \ref{alg:new} takes much less time.

\section{Discussion}\label{sec:summary}

In this work, we prove that the nonproperness set of a parametric polynomial system has some specialization properties
such that the nonproperness set can be computed by an interpolation method.
This method is significantly efficient for
likelihood-equation systems since  the nonproperness set of a likelihood-equation system with respect to the probability variables $p_i$'s is always a finite union of hyperplanes. By this special structure, we can interpolate the factors separately, so a great deal of computing time is saved.  One question is whether it is possible to  filter the ``redundant" factors prior to computation.  For example,
in \eqref{eq:symmetric}, only the nonlinear factor has positive solutions for $\vU$ in ${\mathbb R}_{>0}^6$, which means that none of the linear factors play a role in real root classification. If we can avoid computing factors that admit no positive solutions, then the computational time will be considerably reduced.

\appendix


\section{Testing Models in Table \ref{table:literatureOld}}\label{sec:appendix}

\footnotesize

\begin{model}\citep*[$3\times 3$ Zero-Diagonal Matrix]{EJ2014}\label{ex:l1}
 {\begin{align*}
\det \left[
\begin{array}{ccc}
    0&    p_{12}    & p_{13} \\
    p_{21} &    0 & p_{23}\\
    p_{31} &   p_{32} &  0
\end{array}
\right]=0,\;\;\;p_{12} + p_{13} + p_{21} + p_{23} + p_{31} + p_{32} =1
\end{align*}}
\end{model}

\begin{model}\citep*[Random Censoring Model]{DSS2009}\label{ex:l2}
\[2p_0p_1p_2 + p_1^2p_2 + p_1p_2^2 - p_0^2p_{12} + p_1p_2p_{12}=0, \;\;\; p_0 + p_1 + p_2 + p_{12} = 1\]
\end{model}

\begin{model}\citep*[Grassmannian of $2$-planes in ${\mathbb C}^4$]{SAB2005, EJ2014}\label{ex:l3}
\[p_{12}p_{34}-p_{13}p_{24}+p_{14}p_{23}=0, \;\;\; p_{12} + p_{13} + p_{14} + p_{23} + p_{24} + p_{34} =1\]
\end{model}

\begin{model}\citep*[$3\times 3$ Symmetric Matrix]{SAB2005}\label{ex:l4}
{\begin{equation*}
\det\left[
\begin{array}{cccc}
    2p_{2} &    p_{1}    & p_{3} \\
    p_{1} &    2p_{4}   & p_{5}\\
    p_{3} & p_{5} & 2p_{6}
\end{array}
\right]=0, \;\;\; p_{1} + p_{2} + p_{3} + p_{4} + p_{5} + p_{6} =1
\end{equation*}}
\end{model}

\begin{model}\citep*[$P_{comb}$, Example 15]{SAB2005}\label{ex:l5}
\[q_3-q_5,  \;\;\;q_2-q_5, \;\;\; q_4-q_6, \;\;\;q_5q_7-q_1q_8=0, \;\;\;
p_1 + p_2 + p_3 + p_4 + p_5 + p_6 + p_7 + p_8=1\]
where\\
$q_1 = p_1 + p_2 + p_3 + p_4 + p_5 + p_6 + p_7 + p_8$,
$q_2 = p_1 - p_2 + p_3 - p_4 + p_5 - p_6 + p_7 - p_8$,\\
$q_3 =  p_1 + p_2 - p_3 - p_4 + p_5 + p_6 - p_7 - p_8$,
$q_4 =  p_1 - p_2 - p_3 + p_4 + p_5 - p_6 - p_7 + p_8$,\\
$q_5 =  p_1 + p_2 + p_3 + p_4 - p_5 - p_6 - p_7 - p_8$,
$q_6 =  p_1 - p_2 + p_3 - p_4 - p_5 + p_6 - p_7 + p_8$,\\
$q_7 =  p_1 + p_2 - p_3 - p_4 - p_5 - p_6 + p_7 + p_8$,
$q_8 =  p_1 - p_2 - p_3 + p_4 - p_5 + p_6 + p_7 - p_8$.
\end{model}

\begin{model}\citep*[$3\times 3$ Matrix]{SAB2005}\label{ex:l6}
{\footnotesize \begin{equation*}
\det\left[
\begin{array}{cccc}
    p_{00} &    p_{01}    & p_{02} \\
    p_{10} &    p_{11}   & p_{12}\\
    p_{20} & p_{21} &  p_{22}
\end{array}
\right]=0, \;\;\; p_{00} + p_{01} + p_{02} + p_{10} + p_{11} + p_{12} +  p_{20} + p_{21} + p_{22} =1
\end{equation*}}
\end{model}

\begin{model}\citep*[Bernoulli $3\times 3$ Coin]{SAB2005}\label{ex:l7}
{\footnotesize \begin{equation*}
\det\left[
\begin{array}{cccc}
    12p_{0} &    3p_{1}    & 2p_{2} \\
    3p_{1} &    2p_{2}   & 3p_{3}\\
    2p_{2} & 3p_{3} & 12p_{4}
\end{array}
\right]=0,\;\;\; p_{0} + p_{1} + p_{2} + p_{3} + p_{4}  =1
\end{equation*}}
\end{model}


\begin{model}\citep*[Example 15]{SAB2005}\label{ex:l8}
\[q_2q_7-q_1q_8=0, \;\;\;q_3q_6-q_5q_4=0, \;\;\;
p_1 + p_2 + p_3 + p_4 + p_5 + p_6 + p_7 + p_8=1\]
where\\
$q_1 = p_1 + p_2 + p_3 + p_4 + p_5 + p_6 + p_7 + p_8$,
$q_2 = p_1 - p_2 + p_3 - p_4 + p_5 - p_6 + p_7 - p_8$,\\
$q_3 =  p_1 + p_2 - p_3 - p_4 + p_5 + p_6 - p_7 - p_8$,
$q_4 =  p_1 - p_2 - p_3 + p_4 + p_5 - p_6 - p_7 + p_8$,\\
$q_5 =  p_1 + p_2 + p_3 + p_4 - p_5 - p_6 - p_7 - p_8$,
$q_6 =  p_1 - p_2 + p_3 - p_4 - p_5 + p_6 - p_7 + p_8$,\\
$q_7 =  p_1 + p_2 - p_3 - p_4 - p_5 - p_6 + p_7 + p_8$,
$q_8 =  p_1 - p_2 - p_3 + p_4 - p_5 + p_6 + p_7 - p_8$.
\end{model}

\begin{model}\citep*[Juke-Cantor Model, Example 18]{SAB2005}\label{ex:l9}
\[q_{000}q_{111}^2 - q_{011} q_{101} q_{110}=0, \;\;\; p_{123} + p_{dis} + p_{12} + p_{13} + p_{23}=1\]
where\\
$q_{111} = p_{123} + \frac{p_{dis}}{3} - \frac{p_{12}}{3} - \frac{p_{13}}{3} - \frac{p_{23}}{3}$,
$q_{110} = p_{123} -  \frac{p_{dis}}{3} + p_{12} - \frac{p_{13}}{3} - \frac{p_{23}}{3}$,\\
$q_{101} = p_{123} -  \frac{p_{dis}}{3} - \frac{p_{12}}{3} + p_{13} - \frac{p_{23}}{3}$,
$q_{011} = p_{123} -  \frac{p_{dis}}{3} - \frac{p_{12}}{3} - \frac{p_{13}}{3} + p_{23}$,\\
$q_{000} = p_{123} + p_{dis} + p_{12} + p_{13} + p_{23}$.
\end{model}


\bibliographystyle{cas-model2-names}

\bibliography{cas-refs}

@ARTICLE{ABBGHHNRS2017,
  author  = {Am\'{e}ndola, C. and Bliss, N. and Burke, I. and Gibbons, C. R. and Helmer, M. and Hoşten, S. and Nash, E. D. and Rodriguez, J. I. and Smolkin, D.},
  title   = {The maximum likelihood degree of toric varieties},
  journal = {J. Symbolic Comput.},
  volume  = {92},
  year    = {2019},
  pages   = {222--242}
}

@ARTICLE{AGKMS2024,
  author  = {Am\'{e}ndola, C. and Gustafsson, L. and Kohn, K. and Marigliano, O. and Seigal, A.},
  title   = {Differential Equations for {G}aussian Statistical Models with Rational Maximum Likelihood Estimator},
  journal = {SIAM J. Appl. Algebra Geom.},
  volume  = {8},
  number  = {3},
  year    = {2024},
  pages   = {465--492}
}

@ARTICLE{BPR1996,
  author  = {Basu, S. and Pollack, R. and Roy, M. F.},
  title   = {On the combinatorial and algebraic complexity of quantifier elimination},
  journal = {J. ACM},
  volume  = {43},
  number  = {6},
  year    = {1996},
  pages   = {1002--1045}
}

@ARTICLE{BPRRoadmap,
  author  = {Basu, S. and Pollack, R. and Roy, M. F.},
  title   = {Computing roadmaps of semi-algebraic sets on a variety},
  journal = {J. Amer. Math. Soc.},
  volume  = {3},
  number  = {1},
  year    = {2000},
  pages   = {55--82}
}

@BOOK{BPRBook,
  author    = {Basu, S. and Pollack, R. and Roy, M. F.},
  title     = {Algorithms in Real Algebraic Geometry},
  edition   = {2nd ed.},
  publisher = {Springer-Verlag, Berlin},
  year      = {2006}
}

@ARTICLE{brown2003,
  author  = {Brown, C. W.},
  title   = {{QEPCAD} {B}: a program for computing with semi-algebraic sets using {CAD}s},
  journal = {SIGSAM bulletin},
  volume  = {37},
  number  = {4},
  year    = {2003},
  pages   = {97--108}
}

@ARTICLE{BHR2007,
  author  = {Buot, M.-L. G. and Hoşten, S. and Richards, D.},
  title   = {Counting and locating the solutions of polynomial systems of maximum likelihood equations, II: the Behrens-Fisher problem},
  journal = {Statist. Sinica},
  volume  = {17},
  number  = {4},
  year    = {2007},
  pages   = {1343--1354}
}

@ARTICLE{CHKS2006,
  author  = {Catanese, F. and Hoşten, S. and Khetan, A. and Sturmfels, B.},
  title   = {The maximum likelihood degree},
  journal = {Amer. J. Math.},
  volume  = {128},
  number  = {3},
  year    = {2006},
  pages   = {671--697}
}

@CONFERENCE{CDMMX2010,
  author    = {Chen, C. and Davemport, J. H. and May, J. P. and Maza, M. M. and Xia, B. and Xiao, R.},
  title     = {Triangular decomposition of semi-algebraic systems},
  booktitle = {In Proc. ISSAC'10},
  publisher = {ACM, New York},
  year      = {2010},
  pages     = {187--194}
}

@INCOLLECTION{collins1975,
  author    = {Collins, G. E.},
  title     = {Quantifier elimination for the elementary theory of real closed fields by cylindrical algebraic decomposition},
  booktitle = {Automata Theory and Formal Languages},
  volume    = {33},
  publisher = {Springer Berlin Heidelberg},
  year      = {1975},
  pages     = {134--183}
}

@ARTICLE{ch1991,
  author  = {Collins, G. E. and Hong, H.},
  title   = {Partial cylindrical algebraic decomposition for quantifier elimination},
  journal = {J. Symbolic Comput.},
  volume  = {12},
  number  = {3},
  year    = {1991},
  pages   = {299--328}
}

@BOOK{CLO2015,
  author    = {Cox, D. A. and Little, J. and Oshea, D.},
  title     = {Ideals, Varieties, and Algorithms: An Introduction to Computational Algebraic Geometry and Commutative Algebra},
  edition   = {4th ed.},
  publisher = {Springer, Cham},
  year      = {2015}
}

@ARTICLE{DS1997,
  author  = {Dolzmann, A. and Sturm, T.},
  title   = {Redlog: Computer algebra meets computer logic},
  journal = {ACM SIGSAM Bull.},
  volume  = {31},
  number  = {2},
  year    = {1997},
  pages   = {2--9}
}

@BOOK{DSS2009,
  author    = {Drton, M. and Sturmfels, B. and Sullivant, S.},
  title     = {Lectures on Algebraic Statistics},
  publisher = {Birkhäuser Verlag, Basel},
  year      = {2009}
}

@CONFERENCE{FGb,
  author    = {Faugère, J. C.},
  title     = {FGb: a library for computing Gröbner bases},
  booktitle = {Mathematical Software -- ICMS},
  series    = {Lecture Notes in Comput. Sci.},
  volume    = {6327},
  publisher = {Springer Berlin Heidelberg},
  year      = {2010},
  pages     = {84--87}
}

@BOOK{FJ2008,
  author    = {Fried, M. D. and Jarden, M.},
  title     = {Field Arithmetic},
  edition   = {3rd ed.},
  publisher = {Springer},
  year      = {2008}
}

@ARTICLE{grig88,
  author  = {Grigoriev, D.},
  title   = {Complexity of deciding {T}arski algebra},
  journal = {J. Symbolic Comput.},
  volume  = {5},
  number  = {1--2},
  year    = {1988},
  pages   = {65--108}
}

@ARTICLE{GDP2012,
  author  = {Gross, E. and Drton, M. and Petrović, S.},
  title   = {Maximum likelihood degree of variance component models},
  journal = {Electron. J. Stat.},
  volume  = {6},
  year    = {2012},
  pages   = {993--1016}
}

@CONFERENCE{EJ2014,
  author    = {Gross, E. and Rodriguez, J. I.},
  title     = {Maximum likelihood geometry in the presence of data zeros},
  booktitle = {In Proc. ISSAC'14},
  publisher = {ACM, New York},
  year      = {2014},
  pages     = {232--239}
}

@ARTICLE{HRS,
  author  = {Hauenstein, J. and Rodriguez, J. I. and Sturmfels, B.},
  title   = {Maximum likelihood for matrices with rank constraints},
  journal = {J. Algebr. Stat.},
  volume  = {5},
  number  = {1},
  year    = {2014},
  pages   = {18--38}
}

@ARTICLE{SAB2005,
  author  = {Hoşten, S. and Khetan, A. and Sturmfels, B.},
  title   = {Solving the likelihood equations},
  journal = {Found. Comput. Math.},
  volume  = {5},
  number  = {4},
  year    = {2005},
  pages   = {389--407}
}

@ARTICLE{DV2005,
  author  = {Lazard, D. and Rouillier, F.},
  title   = {Solving parametric polynomial systems},
  journal = {J. Symbolic Comput.},
  volume  = {42},
  number  = {6},
  year    = {2007},
  pages   = {636--667}
}

@ARTICLE{renegar1992-1,
  author  = {Renegar, J.},
  title   = {On the computational complexity and geometry of the first-order theory of the reals, Part {I}},
  journal = {J. Symbolic Comput.},
  volume  = {13},
  number  = {3},
  year    = {1992},
  pages   = {255--299}
}

@ARTICLE{renegar1992-2,
  author  = {Renegar, J.},
  title   = {On the computational complexity and geometry of the first-order theory of the reals, Part {II}},
  journal = {J. Symbolic Comput.},
  volume  = {13},
  number  = {3},
  year    = {1992},
  pages   = {301--327}
}

@ARTICLE{renegar1992-3,
  author  = {Renegar, J.},
  title   = {On the computational complexity and geometry of the first-order theory of the reals, Part {III}},
  journal = {J. Symbolic Comput.},
  volume  = {13},
  number  = {3},
  year    = {1992},
  pages   = {329--352}
}

@CONFERENCE{RT2015,
  author    = {Rodriguez, J. I. and Tang, X.},
  title     = {Data-discriminants of likelihood equations},
  booktitle = {In Proc. ISSAC'15},
  publisher = {ACM, New York},
  year      = {2015},
  pages     = {307--314}
}

@ARTICLE{Tang2017,
  author  = {Rodriguez, J. I. and Tang, X.},
  title   = {A probabilistic algorithm for computing data-discriminants of likelihood equations},
  journal = {J. Symbolic Comput.},
  volume  = {83},
  year    = {2017},
  pages   = {342--364}
}

@CONFERENCE{SS2003,
  author    = {Safey El Din, M. and Schost, E.},
  title     = {Polar varieties and computation of one point in each connected component of a smooth algebraic set},
  booktitle = {In Proc. ISSAC'03},
  publisher = {ACM, New York},
  year      = {2003},
  pages     = {224--231}
}

@ARTICLE{SS2004,
  author  = {Safey EI Din, M. and Schost, E.},
  title   = {Properness defects of projections and computation of at least one point in each connected component of a real algebraic set},
  journal = {Discrete Comput. Geom.},
  volume  = {32},
  number  = {3},
  year    = {2004},
  pages   = {417--430}
}

@CONFERENCE{TWZ2019,
  author    = {Tang, X. and Wolff, T. and Zhao, R.},
  title     = {A new method for computing elimination ideals of likelihood equations},
  booktitle = {In Proc. ISSAC'19},
  publisher = {ACM, New York},
  year      = {2019},
  pages     = {339--346}
}

@BOOK{tarski1951,
  author    = {Tarski, A.},
  title     = {A Decision Method for Elementary Algebra and Geometry},
  edition   = {2nd ed.},
  publisher = {University of California Press, Berkeley and Los Angeles, Calif.},
  year      = {1951}
}

@CONFERENCE{BP2001,  author    = {Yang, L. and Xia, B.},  title     = {Real solution classifications of a class of parametric semi-algebraic systems},  booktitle = {In Proc. the A3L 2005 on Algorithmic Algebra and Logic},  publisher = {Herstellung und Verlag, Norderstedt},  year      = {2005},  pages     = {281--289}}


\end{document}